\documentclass{article}

\usepackage[utf8]{inputenc}
\usepackage[T1]{fontenc}
\usepackage{lmodern}

\usepackage{geometry}
\usepackage{amsmath,amssymb,amsfonts,amsthm}
\usepackage{mathtools}
\usepackage{bm}
\usepackage{mathrsfs}
\usepackage{calrsfs}   
\DeclareMathAlphabet{\mathcal}{OMS}{cmsy}{m}{n}

\usepackage{graphicx}
\usepackage{subcaption}
\usepackage{float}

\usepackage{tikz}
\usetikzlibrary{graphs,arrows.meta,positioning}

\usepackage{algorithm}
\usepackage{algpseudocode}

\usepackage{enumitem}

\newtheorem{theorem}{Theorem}
\newtheorem{lemma}{Lemma}
\newtheorem{proposition}{Proposition}

\theoremstyle{definition}
\newtheorem{definition}{Definition}
\theoremstyle{remark}
\newtheorem{remark}{Remark}

\newcommand{\R}{\mathbb{R}}
\newcommand{\Hcal}{\mathcal{H}}

\newcommand{\calG}{\mathcal{G}}

\DeclareMathOperator{\diag}{diag}

\usepackage[colorlinks=true,
            linkcolor=black,
            citecolor=blue,
            urlcolor=blue]{hyperref}

\title{ Tree-Preconditioned Differentiable Optimization and Axioms as Layers}
\author{
  Yuexin Liao\thanks{
    Division of Physics, Mathematics and Astronomy (PMA),  
    California Institute of Technology.  
    \textsf{yliao@caltech.edu}. 
     Yuexin acknowledges financial support from the Guo \& Zhao Family Research Fellowship 
    and the SURF program at Caltech (2025).
  }
}
\date{October 2025}

\begin{document}

\maketitle
\tableofcontents
\newpage   
\begin{abstract}
We introduce a differentiable framework that embeds the axiomatic structure of Random Utility Models (RUM) directly into deep neural networks. 
Addressing the computational intractability of projecting onto the RUM polytope, defined by exponentially many Block-Marschak inequalities, we establish an isomorphism between RUM consistency and network flow conservation on the Boolean lattice. 
Leveraging this combinatorial structure, we derive a novel Tree-Preconditioned Conjugate Gradient solver. By exploiting the spanning tree of the constraint graph, our preconditioner effectively ``whitens'' the ill-conditioned Hessian spectrum induced by the Interior Point Method barrier, achieving superlinear convergence and scaling to problem sizes,
 previously deemed unsolvable.
We further formulate the projection as a differentiable layer via the Implicit Function Theorem, where the exact Jacobian propagates geometric constraints during backpropagation.
Empirical results demonstrate that this ``Axioms-as-Layers'' paradigm eliminates the structural overfitting inherent in penalty-based methods, enabling models that are jointly trainable, provably rational, and capable of generalizing from sparse data regimes where standard approximations fail.
\end{abstract}
\section{Introduction}

While neural networks excel at high-dimensional function approximation, they fundamentally lack the ability to adhere to structural axioms—such as rationality in economics or conservation laws in physics—unless these constraints are explicitly enforced. Random Utility Models (RUMs) provide the geometric representation of first-principles rationality in choice modeling, expressing choice probabilities as a convex combination of deterministic utility-maximizing rules. This framework forms the theoretical cornerstone of travel demand modeling \cite{McFadden1974,McFadden2000}, product pricing \cite{Berry1995}, and stochastic response modeling in socio-technical infrastructures \cite{Daina2017,Daziano2022}. However, projecting raw neural predictions onto the RUM polytope has historically been computationally intractable beyond trivial scales ($n \approx 5$ alternatives) due to the combinatorial explosion of the defining constraints \cite{KitamuraStoye2018}.

The computational intractability of the RUM projection problem has historically been attributed to its theoretical hardness. The general membership problem is known to be NP-hard\cite{KitamuraStoye2018}. Previous computational approaches predominantly rely on the \textit{vertex representation}, modeling the RUM polytope as the convex hull of $n!$ deterministic preference orderings. This factorial explosion inevitably forces reliance on heuristic approximations or column-generation methods that lack convergence guarantees. We propose a fundamental paradigm shift: characterizing the geometry via its \textit{hyperplane representation}. By enforcing the Block--Marschak inequalities directly, we recast the projection as a Network Quadratic Program (QP). While hyperplane representation bypasses the factorial bottleneck, it exposes a distinct numerical challenge: preconditioning the resulting KKT system. 

In fact, network quadratic programs are ubiquitous in network flow control and statistical estimation \cite{Ahuja1993,Low2014a,Low2014b}. Classical large-scale solvers combine Interior-Point Methods (IPMs) with Krylov subspace methods, whose scalability hinges on preconditioning the KKT system \cite{Gondzio2012,Benzi2005}. Yet, a critical gap exists: while nearly-linear time solvers have been established for Laplacian/SDD systems \cite{SpielmanTeng2014,KoutisMillerPeng2014}, the Schur complement arising from RUM constraints creates a non-Laplacian system with dense, indefinite correlations that defy standard graph-theoretic preconditioners.

We resolve this bottleneck by establishing an isomorphism between the RUM feasibility set and a flow network. We show that after a Block--Marschak M\"obius transform, RUM consistency is equivalent to (i) non-negativity of BM polynomials and (ii) flow conservation on the \textit{Boolean lattice} (the Hasse diagram of subsets). In this representation, the classical normalization constraints become a unit total flow condition. Crucially, by working in a reduced face coordinate system, we eliminate linear equalities entirely, yielding a Symmetric Positive Definite (SPD) Schur complement $\Hcal$ in the IPM:
\[
\Hcal \;=\; B^\top B \;+\; (KB)^\top D\,(KB),\qquad D=\diag(S^{-1}\Lambda),
\]
where $B$ encodes the reduced coordinates, $K$ is the BM transform, and $D$ is the IPM scaling matrix.

\textbf{The Tree-Preconditioner.} To solve systems involving $\Hcal$ efficiently, we exploit the discrete Boolean-lattice flow structure to construct a novel preconditioner $M$. Unlike generic algebraic preconditioners, $M$ is derived explicitly from the combinatorial skeleton of the constraints:
\[
M \;=\; A_m^\top\,\max(D_m,I)\,A_m.
\]
Here, $A_m$ is a square submatrix of $KB$ and $D_m$ captures the dominant entries of the scaling matrix $D$. We prove a graph-theoretic invertibility criterion: \emph{$A_m$ is invertible if and only if the complement of its index set forms a spanning tree in the Boolean lattice.} This insight allows us to replace expensive matrix factorizations with linear-time tree traversals—specifically, forward flow extension for $A_m^{-1}$ and reverse adjoint distribution for $A_m^{-\top}$. Empirically, this topological preconditioning sharply compresses the spectrum of $\Hcal$, yielding a $100\times$ reduction in IPM iterations, with each preconditioned CG iteration costing only $\sim 5\times$ more than the unpreconditioned case.

\textbf{Our Contributions.} By exposing this hidden network structure, we introduce a differentiable framework that makes rationality axioms first-class citizens in deep learning. Specifically:

\begin{enumerate}[label=(\alph*)]
    \item \textbf{A Structure-Exploiting Preconditioner for Non-Laplacian Systems.} We deliver solutions to the RUM nearest-point problem up to $n=20$ alternatives ($N > 10^7 $). We provide, to our knowledge, the first tree/co-tree preconditioner tailored to the non-Laplacian Schur complement of the BM-flow, complete with a graph-theoretic invertibility certificate. This moves computational practice far beyond the previous feasibility–testing regime at $n \approx 5$, enlarging the tractable scale by more than an order of magnitude.

    \item \textbf{Axioms as Differentiable Layers.} We utilize the implicit function theorem to embed the solver as a high-throughput differentiable layer within deep neural networks. Crucially, we cast the converged Schur complement $\Hcal$ as the \emph{common linear operator} for both the forward projection and the backward gradient computation $\Hcal w = \nabla_{\xi^\star}\mathcal{J}$. This unified numerical core delivers a numerically stable, high-throughput layer that enforces rationality by construction, enabling models that are jointly trainable, trustworthy, and interpretable \cite{AmosKolter2017,AgrawalEtAl2019}.

    \item \textbf{Training-Time Numerics: Lazy Updates and Warm-Starts.} We optimize the layer for the deep learning loop by implementing predictor-corrector warm-starts for the inner CG solves and triggering preconditioner rebuilds only under multiplicative changes in $\log D$ or iteration blow-ups. This ensures stable gradients and high throughput in batched training settings.

    \item \textbf{Low-Rank Acceleration via Incomplete Data.} Counter-intuitively, we demonstrate that \textit{missing data accelerates computation}. We show that the solver's complexity is governed by the effective rank of the observation manifold rather than the ambient dimension. The algorithm naturally exploits the low-rank structure of sparse datasets, making it transferable to real-world settings where full choice data is unavailable.
    \item \textbf{Empirical Validation.}  We substantiate our theoretical claims through numerical experiments.
First, using a ``frozen barrier'' protocol, we provide empirical evidence of spectral clustering,
showing that our preconditioner compresses the eigenvalues of the ill-conditioned Hessian
into a compact cluster, enabling superlinear convergence where standard diagonal scaling fails.
Second, in a static stress test with barrier weights spanning eight orders of magnitude ($\kappa \approx 10^8$), we demonstrate that our tree-preconditioner achieves superlinear convergence where standard diagonal scaling (Jacobi) stagnates, empirically proving effective spectral whitening. Third, in a teacher-student learning setup, we demonstrate that our differentiable layer eliminates
structural overfitting. Unlike soft-penalty baselines which fail to satisfy axioms on unseen data,
our approach acts as a strong inductive bias, achieving machine-precision feasibility ($\epsilon \approx 10^{-16}$)
and superior generalization in data-starved regimes.
\end{enumerate}

\section{Predictor-corrector Interior Point Method for Solving the Optimization Problem}
\subsection{Problem Setting}
\begin{itemize}
    \item $X =\{ 0, 1, \dots, n-1 \}$ is the set of all alternatives.
    \item  Let  $N \;:=\; \sum_{D\subseteq X}|D|\;=\;n\,2^{\,n-1}$ denote the number of ordered pairs $(D,x)$ with $x\in D\subseteq X$.
Define the \emph{primal space}
$\mathcal P_\rho\cong\R^{N}$ whose coordinates are
$\rho(D,x)$, the probability of choosing $x$ from $D$.
    \item Let $\mathcal{B} =\{ (D,x): \emptyset \neq D \subseteq X, x \in D, x \neq \max D \}$.
\end{itemize}

BM (Möbius) operator\; 
$ K:\mathcal P_\rho\to\mathcal P_\kappa$

\begin{equation}\label{eq:mobius}
  (\mathcal K\rho)(D,x)
  \;=\;
  \sum_{E\supseteq D}(-1)^{\,|E\setminus D|}\rho(E,x).
\end{equation}
  
Zeta operator $\ K^{-1}:\mathcal P_\kappa\to\mathcal P_\rho$
\begin{equation}\label{eq:zeta}
  (\mathcal K^{-1}\kappa)(D,x)
  \;=\;
  \sum_{E\supseteq D}\kappa(E,x).
\end{equation}

Both maps are linear and mutually inverse; hence each is represented by an
invertible lower-triangular, totally unimodular matrix
$K\in\R^{N\times N}$.

Define a linear operator $B : \mathbb{R}^{\mathcal{B}} \rightarrow \mathbb{R}^N$ as follows:
for a vector $\xi : \mathcal{B} \rightarrow \mathbb{R}$,
\begin{equation}
    B \xi (D,x) := \begin{cases}
        \xi (D,x), & \text{ if } x \neq \max D\\
        - \sum_{y \in D, y \neq \max D} \xi (D,y), & \text{ if } x= \max D,
    \end{cases}
\end{equation}
where $\emptyset \neq D \subseteq X,\, x \in D$.
$(\mathcal{K}\rho)(D,x) = \sum_{E\supseteq D}(-1)^{\,|E\setminus D|}\rho(E,x)$
Define a linear operator $R: \mathbb{R}^N \rightarrow \mathbb{R}^{\mathcal{B}}$ as follows: for a vector $\rho \in \mathbb{R}^N$,
\begin{equation}
    R \rho (D,x) := \rho (D,x),
\end{equation}
where $\emptyset \neq D \subseteq X,\, x \in D,\, x \neq \max D$.

Define $u \in \mathbb{R}^N$ as
\begin{equation}
    u (D,x) = \begin{cases}
        1, & \text{ if } x =\max D,\\
        0, & \text{ if } x \neq \max D,
    \end{cases}
\end{equation}
where $\emptyset \neq D \subseteq X,\, x \in D$.

Then for a choice probability vector $\rho \in \mathbb{R}^N$,
\begin{equation}
    \rho =BR\rho +u.
\end{equation}
The problem is to project any choice vector $\rho$ to the nearest point on random utility model and seek their minimum $L_2$ distance.
The corresponding optimization problem is the quadratic programming with equality and inequality constrains
\begin{equation}
    \min (\rho -\hat{\rho})^T (\rho -\hat{\rho}) \text{ s.t. } C \rho =1,\, K \rho \geq 0
\end{equation}

\subsection{Problem Formulation}
Let $\xi =R \rho$, we can transform the above problem into a new formulation:
\begin{equation}
    \min (\rho -\hat{\rho})^T (\rho -\hat{\rho}) \text{ s.t. } C \rho =1,\, K \rho \geq 0
\end{equation} is equivalent to the following quadratic programming with only inequality constrains
\begin{equation}
    \min (B \xi +u -\hat{\rho})^T (B \xi +u -\hat{\rho}) \text{ s.t. } K(B \xi +u) \geq 0
\end{equation}
because $B \xi +u$ automatically satisfies $C (B \xi +u) =1$. By scaling and canceling some constant, we get the following quadratic programming problem:
\begin{equation}\label{00002ii3}
    \min \frac{1}{2} \xi^T B^T B \xi + (B^T u -B^T \hat{\rho})^T \xi \text{ s.t. } KB \xi +Ku \geq 0
\end{equation}

By definition of $B$, the linear operator $B^T : \mathbb{R}^N \rightarrow \mathbb{R}^\mathcal{B}$ is as
\begin{equation}
    B^T \rho (D,x) =\rho (D,x) -\rho (D, \max D),
\end{equation}
where $\rho \in \mathbb{R}^N,\, \emptyset \neq D \subseteq X,\, x \in D,\, x \neq \max D$. Thus, $B^T u =- \mathbf{1}$.

Let $b =Ku$. Since $u$ can be seen as the deterministic choice vector with the preference order $0 \prec 1 \prec \cdots \prec n-1$. Then $b= Ku$ is the BM polynomial vector with this preference order, which is a flow on the path from $X$ to $\emptyset$ by ruling out the best alternative each time. Thus,
\begin{equation}
    b (D,x) =\begin{cases}
        1, &\text{ if } D =\{ 0, 1, \dots, x\},\\
        0, &\text{ otherwise.}
    \end{cases}
\end{equation}

Hence \eqref{00002ii3} can be written as
\begin{equation}\label{wced}
    \min \frac{1}{2} \xi^T B^T B \xi - (B^T \hat{\rho} +\mathbf{1})^T \xi \text{ s.t. } KB \xi +b \geq 0.
\end{equation}
This posesses the form of quadratic programming problem in [Numerical Optimization, Wright, (16.54)], so we apply the predictor-corrector interior point method [Numerical Optimization, Wright, Algorithm 16.4]:
\\
Choose $\xi_0 \in \mathbb{R}^\mathcal{B}$ with $KB \xi_0 +b \geq 0$, $s_0 \in \mathbb{R}_+^N$, and $\lambda_0 \in \mathbb{R}_+^N$.
\\
For $k=0, 1, 2, \dots$,

let $S_k =diag (s_k)$ and $\Lambda_k =diag (\lambda_k)$.

Solve
\begin{equation}\label{ef999su}
\begin{pmatrix}
B^T B & 0 & -B^T K^T \\
KB & -I & 0 \\
0 & \Lambda_k & S_k 
\end{pmatrix}
\begin{pmatrix}
\Delta \xi^{\text{aff}} \\
\Delta s^{\text{aff}} \\
\Delta \lambda^{\text{aff}} 
\end{pmatrix} =\begin{pmatrix}
-B^T B \xi_k +B^T K^T \lambda_k +B^T \hat{\rho} +\mathbf{1} \\
-KB \xi_k +s_k -b \\
- \Lambda_k S_k \mathbf{1}
\end{pmatrix}.
\end{equation}

Calculate $\mu =s_k^T \lambda_k /N$.

Calculate $\hat{\alpha}_{\text{aff}} =\max \{ \alpha \in (0,1] : (s_k, \lambda_k) +\alpha (\Delta s^{\text{aff}}, \Delta \lambda^{\text{aff}}) \geq 0 \}$

Calculate $\mu_{\text{aff}} =(s_k +\hat{\alpha}_{\text{aff}} \Delta s^{\text{aff}})^T (\lambda_k +\hat{\alpha}_{\text{aff}} \Delta \lambda^{\text{aff}}) /N$

Calculate $\sigma =(\mu_{\text{aff}} /\mu)^3$

Solve
\begin{equation}\label{ef999su2}
\begin{pmatrix}
B^T B & 0 & -B^T K^T \\
KB & -I & 0 \\
0 & \Lambda_k & S_k 
\end{pmatrix}
\begin{pmatrix}
\Delta \xi \\
\Delta s \\
\Delta \lambda
\end{pmatrix} =\begin{pmatrix}
-B^T B \xi_k +B^T K^T \lambda_k +B^T \hat{\rho} +\mathbf{1} \\
-KB \xi_k +s_k -b \\
- \Lambda_k S_k \mathbf{1} -\Delta \Lambda^{\text{aff}} \Delta S^{\text{aff}} \mathbf{1} +\sigma \mu \mathbf{1}
\end{pmatrix},
\end{equation}
where $\Delta \Lambda^{\text{aff}} =diag (\Delta \lambda^{\text{aff}})$ and $\Delta S^{\text{aff}} =diag (\Delta s^{\text{aff}})$.

Choose $\tau_k \in (0,1)$ (we choose $\tau_k =0.995$ currently).

Calculate $\hat{\alpha} =\max \{ \alpha \in (0,1] : (s_k, \lambda_k) +\alpha (\Delta s, \Delta \lambda) \geq (1- \tau_k) (s_k, \lambda_k) \}$.

Set $(\xi_{k+1}, s_{k+1}, \lambda_{k+1}) =(\xi_k, s_k, \lambda_k) +\hat{\alpha} (\Delta \xi, \Delta s, \Delta \lambda)$.
\\
End for.

\subsection{Solving the linear equations}

The hardest step in the above algorithm is to solve \eqref{ef999su} and \eqref{ef999su2}. They are both of the form
\begin{equation}\label{genne}
\begin{pmatrix}
B^T B & 0 & -B^T K^T \\
KB & -I & 0 \\
0 & \Lambda_k & S_k 
\end{pmatrix}
\begin{pmatrix}
\Delta \xi \\
\Delta s \\
\Delta \lambda
\end{pmatrix} =\begin{pmatrix}
b_1 \\
b_2 \\
b_3
\end{pmatrix}.
\end{equation}

The 2nd row is
\begin{equation}
    KB \Delta \xi -\Delta s =b_2,
\end{equation}
which implies
\begin{equation}\label{8r8e}
    \Delta s =KB \Delta \xi -b_2.
\end{equation}

The 3rd row is
\begin{equation}
    \Lambda_k \Delta s +S_k \Delta \lambda =b_3,
\end{equation}
which implies
\begin{equation}\label{2rh9wh}
    \Delta \lambda =S_k^{-1} (b_3 -\Lambda_k \Delta s) =S_k^{-1} (b_3 -\Lambda_k (KB \Delta \xi -b_2)).
\end{equation}
Substituting \eqref{2rh9wh} into the 1st row we get
\begin{equation}\label{w99adfsv}
    (B^T B +B^T K^T S_k^{-1} \Lambda_k KB) \Delta \xi =b_1 +B^T K^T S_k^{-1} (b_3 +\Lambda_k b_2).
\end{equation}

The coefficient matrix $H  := B^T B +B^T K^T S_k^{-1} \Lambda_k KB$ of \eqref{w99adfsv} is positive definite symmetric, so we can use CG method to solve \eqref{w99adfsv}, and then use \eqref{8r8e} and \eqref{2rh9wh} to calculate the solution of \eqref{genne}.
\\
\\

\textbf{Remark:}

\begin{itemize}
    \item CG is a iterative method to approach the solution to positive definite symmetric linear equations like \eqref{w99adfsv}. Without preconditioners, the time of each iteration is slightly longer than the time of an $H$ operation ($v \mapsto Hv$). The number of iterations we need depends on the accuracy of the solution we need and the convergence rate. From [Numerical Optimization, Wright, Section 5.1], roughly speaking, the more clustered the distribution of all eigenvalues of $H$ is, the faster CG converges. [Numerical Optimization, Wright, Section (5.36)] gives an ``often overestimated'' estimation of the convergence rate:
    \begin{equation}\label{rarf}
        \| x_k -x^* \|_H \leq 2 \left( \frac{\sqrt{\kappa (H)} -1}{\sqrt{\kappa (H)} +1} \right)^k \| x_0 -x^* \|_H,
    \end{equation}
    where $x^*$ is the exact solution, $x_0$ is the initial conjectured solution, $x_k$ is the result after $k$ CG iterations, $\| x \|_H := x^T H x$, and $\kappa (H)$ is the condition number of $H$, defined as $\lambda_1 (H)/ \lambda_n (H)$, where $\lambda_1$ and $\lambda_n$ are the largest and the smallest eigenvalues of $H$.
    \item The reason why we work on a reduced space $\mathbb{R}^\mathcal{B}$ and cancel the equality constrains is to make the linear equation positive definite and thus be able to use CG to solve. If we keep the equality constrains, the coefficient of the linear equation will be
    \begin{equation}\label{wefhoi}
    \begin{pmatrix}
    I + K^T D K & C^T \\
    C & 0
    \end{pmatrix}
    \end{equation}
    which is indefinite symmetric. A similar iterative method to solve indefinite symmetric linear equations is MINRES, with a similar cost of each iteration as CG and a slower convergence.
\end{itemize}

\section{A Preconditioner for the CG Method}

\subsection{Motivation and Design Principle}

As $(\xi_k, s_k, \lambda_k)$ goes closer and closer to the solution of \eqref{wced}, some diagonal elements of $S_k^{-1} \Lambda_k$ are very small and some are very large (experiment shows that for a ramdom choice probability vector $\hat{\rho}$, with high probability, $S_k^{-1} \Lambda_k$ has more large elements than small elements, and the portion of small elements does not approach 0). The consequence is that $H= B^T B +B^T K^T S_k^{-1} \Lambda_k KB$ has many very large eigenvalues. In terms of $\kappa (H) =\lambda_1 /\lambda_n$, $\lambda_1 (H)$ is very large, and $\lambda_n (H)$ not so small because it is bounded below by the smallest eigenvalue of $B^T B$, so $\kappa (H)$ is large, and thus CG need many iterations to converge. (For \eqref{wefhoi}, the largest eigenvalue is large and the smallest eigenvalue is small, the condition number is even larger, so MINRES need even more iterations to converge.)

To solve this problem, we need to find a preconditioner $M$. With a preconditioner $M$, the time of each CG iteration is about the time of one $H$ operator and one $M^{-1}$ operator, and the convergence rate is \eqref{rarf} by replacing $\kappa (H)$ with $\kappa (L^{-T} H L)$, where $M =L^T L$. Roughly, the more $M$ is close to $H$, the smaller $\kappa (L^{-T} H L)$ is, the faster CG converges.

The idea to find a preconditioner $M$ is
\begin{itemize}
    \item $v \mapsto M^{-1} v$ is easy to calculate.
    \item $M$ is a good approximation of $H$.
\end{itemize}

Our preconditioner is
\begin{equation}\label{fspvsf}
    M = A_m^{T} \max (D_m, I) A_m,
\end{equation}
where $A_m, D_m \in GL_{|\mathcal{B}|} (\mathbb{R})$, $D_m$ is a diagonal submatrix of $S_k^{-1} \Lambda_k$, $\max (D_m, I)$ is the elementwise maximum of $D_m$ and the identity matrix with the same order, and $A_m$ is an invertible submatrix of $KB$.

Let $D= S_k^{-1} \Lambda_k$.

Note that the rows of $D$ and $KB$ can be  indexed by $(E,x)$, where $x \in E \subseteq X$.
Let $P \subseteq \{ (E,x) : x \in E \subseteq X \}$ be the subset of index such that the diagonal of $D_m$ consists of all diagonal entries of $D$ whose index is in $P$, and $A_m$ consists of all rows of $KB$ whose index is in $P$. We need to choose $P$ such that
\begin{itemize}
    \item $A_m$ is a square matrix, i.e., $|P| =N -(2^n -1)$.
    \item $A_m$ is invertible and $A_m^{-1}$ operator is easy to calculate. 
    \\
    This ensures that $M^{-1}$ operator, which is the composition of $A_m^{-1},\, \max (D_m, I)^{-1},\, A_m^{-T}$ operators, is easy to calculate. (It seems there are no general ways to calculate the operator transpose to some given operator, but I think this is often doable by analyzing the construction of the given operator.)
    \item Diagonal entries of $D_m$ are as large as possible among all diagonal entries in $D$.
    \\
    This relates to the requirement that $M$ is a good approximation of $H$. Recall that our goal is to make $\kappa (L^{-T} H L)$ relatively small to make convergence faster, instead of make $\kappa (L^{-T} H L)$ close to 1, which I think is impossible. Thus, we need $M$ to be close to the part of $H$ which leads to its large eigenvalues. Thus, to approximate
    \begin{equation}
        H= B^T B +(KB)^T D KB
    \end{equation}
    we can ignore $B^T B$, which does not lead to large eigenvalues, and replace $D$ with $\max (D,I)$ to avoid too small eigenvalues. This gives a good approximation of $H$:
    \begin{equation}\label{wf00n23ef}
        (KB)^T \max (D,I) KB.
    \end{equation}
    However, the inverse operator of \eqref{wf00n23ef} seems not easy to calculate because $KB$ is not a square matrix. Thus, we take square submatrix $\max (D_m, I)$ of $\max (D,I)$ and $A_m$ of $KB$ and construct $M$ as \eqref{fspvsf}. To make $M$ close to \eqref{wf00n23ef}, we want the diagonal entries of $\max (D_m, I)$ consists of large diagonal entries of $\max (D,I)$. If this is realized, then I think the approximation is good because the size of $\max (D_m, I)$ is only slightly smaller than the size of $\max (D,I)$ when $n$ is large, which ensures that $\max (D_m, I)$ captures most of the large diagonal entries of $\max (D,I)$.
\end{itemize}

\subsection{The Spanning Tree Condition for Preconditioner Invertibility}

\begin{definition}[The Boolean Lattice Graph $\calG$]
Let $\calG=(V,E)$ be an undirected graph where:
\begin{itemize}
    \item The vertex set $V$ is the power set of $X$, $|V|=2^n$.
    \item The edge set $E$ is indexed by $N = n2^{n-1}$ pairs $(D,x)$, where $x \in D \subseteq X$. The edge indexed by $(D,x)$ connects vertex $D$ and $D \setminus \{x\}$.
\end{itemize}
\end{definition}

\begin{remark}
    The edge of $\mathcal{G}$ is undirected, but a flow on some edge is directed. For example, 1 amount of a flow on an edge $(D,x)$ from $D$ to $D \setminus \{ x \}$ is equivalent to $-1$ amount of flow on the edge from $D \setminus \{ x \}$ to $D$. In the following definition, a flow on $\mathcal{G}$ is represented by a vector $\bm{\kappa} \in \mathbb{R}^E$. The amount of flow on the edge $(D,x)$ is $|\bm{\kappa} (D,x)|$. The $|\bm{\kappa} (D,x)|$ amount of flow on this edge is from $D$ to $D \setminus \{ x \}$ if $\bm{\kappa} (D,x) >0$, and it is from $D \setminus \{ x \}$ to $D$ if $\bm{\kappa} (D,x) <0$.
\end{remark}

\begin{definition}[Flows on $\calG$]
A vector $\bm{\kappa} \in \mathbb{R}^E$ is a \textbf{flow} if it satisfies the conservation law at every vertex $D \in V \setminus \{\emptyset, X\}$:
\begin{equation}\label{veu98euq}
    \sum_{x \in D} \kappa(D,x) = \sum_{y \notin D} \kappa(D \cup \{y\}, y).
\end{equation}
The \textbf{total flow} of $\kappa$ is defined as $\sum_{x \in X} \kappa(X,x)$.
\end{definition}

\begin{theorem}
A vector $\bm{\rho} \in \mathbb{R}^N$ is a valid choice probability vector (i.e., it satisfies the normalization constraints $C\bm{\rho}=\mathbf{1}$) if and only if $\bm{\kappa} = K\bm{\rho}$ is a flow on $\calG$ with a total flow of 1.
\end{theorem}

Thus, $\rho \in Image B$ if and only if $\rho$ satisfies $\sum_{x \in D} \rho (D,x) =0$ for all $D$ if and only if $K \rho$ is a flow with total flow 0. 
Thus, $Image KB$ is the space of all flows with total flow 0.
Moreover, for each flow $\kappa$ with total flow 0, there is a unique preimage $(KB)^{-1} \kappa$, which is $RK^{-1} \kappa$.

Recall that $P \subset \{(E,x)\}$ is a to be selected index set of size $|\mathcal{B}| = N - (2^n-1)$ and $A_m : \mathbb{R}^\mathcal{B} \rightarrow \mathbb{R}^P$ is such an operator:
\begin{equation}\label{viaveanqbfp}
    (A_m \xi) (D,x) = KB \xi (D,x),
\end{equation}
where $(D,x) \in P$.

\begin{theorem}
\label{thm:flow}
$A_m$ is invertible if and only if $P^c =E \setminus P$ forms a spanning tree of the graph $\mathcal{G}$.
\end{theorem}

\begin{proof}
The operator $KB$ gives a 1 to 1 correspondence from $\mathbb{R}^\mathcal{B}$ to $image KB$, which is the space of all flows on $\mathcal{G}$ with total flow 0.
Thus, by \eqref{viaveanqbfp}, $A_m$ is invertible if and only if
for each $v \in \mathbb{R}^P$, there is a unique extension $\kappa \in \mathbb{R}^N$ of $v$ which is a flow with total flow 0.

Consider the case that $P^c$ contains a cycle $T$. Suppose $v \in \mathbb{R}^P$ and there exists a flow $\kappa$ with total flow 0 whose restriction on $P$ is $v$. Suppose $\tau$ is a non-zero cyclic flow on the cycle $T$. Then $\kappa +\tau$ is a flow with total flow 0. Since $\tau$ is a flow on $T \subseteq P^c$, the restriction of $\tau$ onto $P$ is 0, so the restriction of $\kappa +\tau$ onto $P$ is still $v$. We find another extension $\kappa +\tau$ of $v$ which is a flow of total flow 0, so $A_m$ is not invertible.

If $A_m$ is invertible, $P^c$ must not contain a cycle, so $P^c$ is a forest. Since $|P^c| =2^n -1= |V|-1$, it has to be a spanning tree of $\mathcal{G}$.

Assume $P^c$ is a spanning tree. The values of flow on tree edges can be uniquely determined from the values on co-tree $P$ edges by starting at the leaves of the tree and applying the zero-flow condition  iteratively until the root is reached. The algorithm is constructive, as shown below.
\end{proof}

Now we assume $P^c$ is a spanning tree, then for each $v \in \mathbb{R}^P$, there is a unique extension $\kappa \in \mathbb{R}^N$ of $v$ which is a flow with total flow 0. We denote this extension operator by $L_{\text{ext}}$, i.e., the unique 0-flow extension of $v \in \mathbb{R}^P$ is $L_{\text{ext}} v$.

\textbf{Algorithmic Implementation of $L_{\text{ext}}$ and its Adjoint}

The forward operator $L_{\text{ext}}$ is an information aggregation process on the spanning tree, while its adjoint $L_{\text{ext}}^T$ is an information distribution process.
\begin{algorithm}
  \caption{Forward Pass $\kappa = L_{\mathrm{ext}}(v_P)$}
  \label{alg:forward-pass}
  \begin{algorithmic}[1]
    \Require $v_P \in \mathbb{R}^P$
    \Ensure Full zero-flow $\kappa \in \mathbb{R}^N$
    \State Initialize $\kappa \in \mathbb{R}^N$ by setting
    $\kappa_e \gets (v_P)_e$ for all $e \in P$, and marking
    $\kappa_e$ as ``unknown'' for all $e \in T := P^c$.
    \State $T' \gets T$.
    \While{$T' \neq \emptyset$}
      \State Find a vertex $u \in V$ that is a leaf of the current tree $T'$.
      \State Let $e_u \in T'$ be the unique tree edge incident to $u$.
      \State Enforce the zero-flow condition at vertex $u$:
      \[
        \sum_{e \text{ incident to } u} \sigma_{ue}\,\kappa_e = 0,
      \]
      where $\sigma_{ue} \in \{+1,-1\}$ is given by~\eqref{veu98euq}.
      \State In this sum, only $\kappa_{e_u}$ is unknown. Solve for it:
      \[
        \kappa_{e_u}
        \gets
        - \sigma_{ue_u}
        \sum_{\substack{e \neq e_u \\ e \text{ incident to } u}}
        \sigma_{ue}\,\kappa_e.
      \]
      \State Remove $e_u$ from $T'$.
    \EndWhile
    \State \Return $\kappa$
  \end{algorithmic}
\end{algorithm}

\newpage
\begin{remark}[Geometric Duality and Complexity Separation]
\label{rem:complexity_duality}
Our framework and the seminal work of Smeulders et al. (\cite{SmeuldersCherchyeDeRock2021}) represent dual geometric approaches to the RUM projection problem, distinguished by a fundamental phase transition in computational complexity. Smeulders et al. operate on the \textit{Vertex Representation} ($\mathcal{P}_V = \operatorname{conv}(\Pi_n)$), where the primal variable is a mixture over $n!$ rankings. The bottleneck of this approach is the column generation pricing subproblem, which is equivalent to the \textit{Maximum Weight Rank Aggregation} problem and is known to be NP-hard \cite{dwork2001rank, jagabathula2019limit}. Consequently, the complexity lower bound scales factorially, $\Omega(n!)$, rendering $n \approx 10$ intractable.

In contrast, our framework leverages the \textit{Hyperplane Representation} ($\mathcal{P}_H = \{K\rho \ge 0\}$), transforming the combinatorial search into a continuous quadratic program. While the constraint count $N = n2^{n-1}$ scales exponentially, this represents a massive dimensionality reduction from the factorial scaling of $\mathcal{P}_V$ (e.g., at $n=20$, $N \approx 10^7$ vs. $n! \approx 10^{18}$). The historic barrier to the $\mathcal{P}_H$ approach has been the exponentially ill-conditioned Hessian arising from the barrier method. Our contribution is to resolve this \textit{numerical complexity} via the Tree-Preconditioner ( \ref{thm:flow} ), effectively reducing the problem to a sequence of linear solves in $\mathcal{P}$ (Polynomial time in $N$). This shift from combinatorial NP-hardness to numerical linear algebra enables the scalability to $n=20$ and the differentiability required for the deep learning layers in Section 4.
\end{remark}

\begin{figure}[h!] 
 \centering 
 
 \begin{minipage}{0.48\textwidth}
     \centering
     \begin{tikzpicture}[scale=1.2, font=\small] 
         \node (0) at (0,0) {$\emptyset$}; 
         \node (1) at (-2,1.5) {$\{0\}$}; 
         \node (2) at (0,1.5) {$\{1\}$}; 
         \node (3) at (2,1.5) {$\{2\}$}; 
         \node (12) at (-2,3) {$\{0,1\}$}; 
         \node (13) at (0,3) {$\{0,2\}$}; 
         \node (23) at (2,3) {$\{1,2\}$}; 
         \node (123) at (0,4.5) {$\{0,1,2\}$}; 
         \draw [thick, red] (1) -- node[pos=0.65, font=\footnotesize, fill=white, inner sep=1pt] {$\kappa (e_1) =?$} (0);
         \draw [dashed, blue] (2) -- node[pos=0.3, font=\footnotesize, fill=white, inner sep=1pt] {$\kappa (e_2) =5$} (0);
         \draw [dashed, blue] (3) -- node[pos=0.65, font=\footnotesize, fill=white, inner sep=1pt] {$\kappa (e_3) =-3$} (0);
         \draw [dashed, blue] (12) -- node[midway, font=\footnotesize, fill=white, inner sep=1pt] {$\kappa (e_4) =0$} (1);
         \draw [thick, red] (13) -- node[pos=0.8, font=\footnotesize, fill=white, inner sep=1pt] {$\kappa (e_5) =?$} (1);
         \draw [thick, red] (12) -- node[pos=0.2, font=\footnotesize, fill=white, inner sep=1pt] {$\kappa (e_6) =?$} (2);
         \draw [thick, red] (23) -- node[pos=0.7, left, font=\footnotesize, fill=white, inner sep=1pt] {$\kappa (e_7) =?$} (2);
         \draw [thick, red] (13) -- node[pos=0.3, left, font=\footnotesize, fill=white, inner sep=1pt] {$\kappa (e_8) =?$} (3);
         \draw [thick, red] (23) -- node[midway, font=\footnotesize, fill=white, inner sep=1pt] {$\kappa (e_9) =?$} (3);
         \draw [dashed, blue] (123) -- node[pos=0.35, font=\footnotesize, fill=white, inner sep=1pt] {$\kappa (e_{10}) =2$} (12);
         \draw [thick, red] (123) -- node[pos=0.7, font=\footnotesize, fill=white, inner sep=1pt] {$\kappa (e_{11}) =?$} (13);
         \draw [dashed, blue] (123) -- node[pos=0.35, font=\footnotesize, fill=white, inner sep=1pt] {$\kappa (e_{12}) =-3$} (23);
     \end{tikzpicture} 
 \end{minipage}
 \hfill 
 \begin{minipage}{0.48\textwidth}
     \centering
     \begin{tikzpicture}[scale=1.2, font=\small] 
         \node (0) at (0,0) {$\emptyset$}; 
         \node (1) at (-2,1.5) {$\{0\}$}; 
         \node (2) at (0,1.5) {$\{1\}$}; 
         \node (3) at (2,1.5) {$\{2\}$}; 
         \node (12) at (-2,3) {$\{0,1\}$}; 
         \node (13) at (0,3) {$\{0,2\}$}; 
         \node (23) at (2,3) {$\{1,2\}$}; 
         \node (123) at (0,4.5) {$\{0,1,2\}$}; 
         \draw [dashed, blue] (1) -- node[pos=0.65, font=\footnotesize, fill=white, inner sep=1pt] {$\kappa (e_1) =-2$} (0);
         \draw [dashed, blue] (2) -- node[pos=0.3, font=\footnotesize, fill=white, inner sep=1pt] {$\kappa (e_2) =5$} (0);
         \draw [dashed, blue] (3) -- node[pos=0.65, font=\footnotesize, fill=white, inner sep=1pt] {$\kappa (e_3) =-3$} (0);
         \draw [dashed, blue] (12) -- node[midway, font=\footnotesize, fill=white, inner sep=1pt] {$\kappa (e_4) =0$} (1);
         \draw [thick, red] (13) -- node[pos=0.8, font=\footnotesize, fill=white, inner sep=1pt] {$\kappa (e_5) =?$} (1);
         \draw [thick, red] (12) -- node[pos=0.2, font=\footnotesize, fill=white, inner sep=1pt] {$\kappa (e_6) =?$} (2);
         \draw [thick, red] (23) -- node[pos=0.7, left, font=\footnotesize, fill=white, inner sep=1pt] {$\kappa (e_7) =?$} (2);
         \draw [thick, red] (13) -- node[pos=0.3, left, font=\footnotesize, fill=white, inner sep=1pt] {$\kappa (e_8) =?$} (3);
         \draw [thick, red] (23) -- node[midway, font=\footnotesize, fill=white, inner sep=1pt] {$\kappa (e_9) =?$} (3);
         \draw [dashed, blue] (123) -- node[pos=0.35, font=\footnotesize, fill=white, inner sep=1pt] {$\kappa (e_{10}) =2$} (12);
         \draw [thick, red] (123) -- node[pos=0.7, font=\footnotesize, fill=white, inner sep=1pt] {$\kappa (e_{11}) =?$} (13);
         \draw [dashed, blue] (123) -- node[pos=0.35, font=\footnotesize, fill=white, inner sep=1pt] {$\kappa (e_{12}) =-3$} (23);
     \end{tikzpicture}
 \end{minipage}
 
\end{figure}

\begin{figure}[h!] 
 \centering 
 
 \begin{minipage}{0.48\textwidth}
     \centering
     \begin{tikzpicture}[scale=1.2, font=\small] 
         \node (0) at (0,0) {$\emptyset$}; 
         \node (1) at (-2,1.5) {$\{0\}$}; 
         \node (2) at (0,1.5) {$\{1\}$}; 
         \node (3) at (2,1.5) {$\{2\}$}; 
         \node (12) at (-2,3) {$\{0,1\}$}; 
         \node (13) at (0,3) {$\{0,2\}$}; 
         \node (23) at (2,3) {$\{1,2\}$}; 
         \node (123) at (0,4.5) {$\{0,1,2\}$}; 
         \draw [dashed, blue] (1) -- node[pos=0.65, font=\footnotesize, fill=white, inner sep=1pt] {$\kappa (e_1) =-2$} (0);
         \draw [dashed, blue] (2) -- node[pos=0.3, font=\footnotesize, fill=white, inner sep=1pt] {$\kappa (e_2) =5$} (0);
         \draw [dashed, blue] (3) -- node[pos=0.65, font=\footnotesize, fill=white, inner sep=1pt] {$\kappa (e_3) =-3$} (0);
         \draw [dashed, blue] (12) -- node[midway, font=\footnotesize, fill=white, inner sep=1pt] {$\kappa (e_4) =0$} (1);
         \draw [dashed, blue] (13) -- node[pos=0.8, font=\footnotesize, fill=white, inner sep=1pt] {$\kappa (e_5) =-2$} (1);
         \draw [thick, red] (12) -- node[pos=0.2, font=\footnotesize, fill=white, inner sep=1pt] {$\kappa (e_6) =?$} (2);
         \draw [thick, red] (23) -- node[pos=0.7, left, font=\footnotesize, fill=white, inner sep=1pt] {$\kappa (e_7) =?$} (2);
         \draw [thick, red] (13) -- node[pos=0.3, left, font=\footnotesize, fill=white, inner sep=1pt] {$\kappa (e_8) =?$} (3);
         \draw [thick, red] (23) -- node[midway, font=\footnotesize, fill=white, inner sep=1pt] {$\kappa (e_9) =?$} (3);
         \draw [dashed, blue] (123) -- node[pos=0.35, font=\footnotesize, fill=white, inner sep=1pt] {$\kappa (e_{10}) =2$} (12);
         \draw [thick, red] (123) -- node[pos=0.7, font=\footnotesize, fill=white, inner sep=1pt] {$\kappa (e_{11}) =?$} (13);
         \draw [dashed, blue] (123) -- node[pos=0.35, font=\footnotesize, fill=white, inner sep=1pt] {$\kappa (e_{12}) =-3$} (23);
     \end{tikzpicture} 
 \end{minipage}
 \hfill 
 \begin{minipage}{0.48\textwidth}
     \centering
     \begin{tikzpicture}[scale=1.2, font=\small] 
         \node (0) at (0,0) {$\emptyset$}; 
         \node (1) at (-2,1.5) {$\{0\}$}; 
         \node (2) at (0,1.5) {$\{1\}$}; 
         \node (3) at (2,1.5) {$\{2\}$}; 
         \node (12) at (-2,3) {$\{0,1\}$}; 
         \node (13) at (0,3) {$\{0,2\}$}; 
         \node (23) at (2,3) {$\{1,2\}$}; 
         \node (123) at (0,4.5) {$\{0,1,2\}$}; 
         \draw [dashed, blue] (1) -- node[pos=0.65, font=\footnotesize, fill=white, inner sep=1pt] {$\kappa (e_1) =-2$} (0);
         \draw [dashed, blue] (2) -- node[pos=0.3, font=\footnotesize, fill=white, inner sep=1pt] {$\kappa (e_2) =5$} (0);
         \draw [dashed, blue] (3) -- node[pos=0.65, font=\footnotesize, fill=white, inner sep=1pt] {$\kappa (e_3) =-3$} (0);
         \draw [dashed, blue] (12) -- node[midway, font=\footnotesize, fill=white, inner sep=1pt] {$\kappa (e_4) =0$} (1);
         \draw [dashed, blue] (13) -- node[pos=0.8, font=\footnotesize, fill=white, inner sep=1pt] {$\kappa (e_5) =-2$} (1);
         \draw [dashed, blue] (12) -- node[pos=0.2, font=\footnotesize, fill=white, inner sep=1pt] {$\kappa (e_6) =2$} (2);
         \draw [thick, red] (23) -- node[pos=0.7, left, font=\footnotesize, fill=white, inner sep=1pt] {$\kappa (e_7) =?$} (2);
         \draw [thick, red] (13) -- node[pos=0.3, left, font=\footnotesize, fill=white, inner sep=1pt] {$\kappa (e_8) =?$} (3);
         \draw [thick, red] (23) -- node[midway, font=\footnotesize, fill=white, inner sep=1pt] {$\kappa (e_9) =?$} (3);
         \draw [dashed, blue] (123) -- node[pos=0.35, font=\footnotesize, fill=white, inner sep=1pt] {$\kappa (e_{10}) =2$} (12);
         \draw [thick, red] (123) -- node[pos=0.7, font=\footnotesize, fill=white, inner sep=1pt] {$\kappa (e_{11}) =?$} (13);
         \draw [dashed, blue] (123) -- node[pos=0.35, font=\footnotesize, fill=white, inner sep=1pt] {$\kappa (e_{12}) =-3$} (23);
     \end{tikzpicture}
 \end{minipage}
 
\end{figure}

\begin{figure}[h!] 
 \centering 
 
 \begin{minipage}{0.48\textwidth}
     \centering
     \begin{tikzpicture}[scale=1.2, font=\small] 
         \node (0) at (0,0) {$\emptyset$}; 
         \node (1) at (-2,1.5) {$\{0\}$}; 
         \node (2) at (0,1.5) {$\{1\}$}; 
         \node (3) at (2,1.5) {$\{2\}$}; 
         \node (12) at (-2,3) {$\{0,1\}$}; 
         \node (13) at (0,3) {$\{0,2\}$}; 
         \node (23) at (2,3) {$\{1,2\}$}; 
         \node (123) at (0,4.5) {$\{0,1,2\}$}; 
         \draw [dashed, blue] (1) -- node[pos=0.65, font=\footnotesize, fill=white, inner sep=1pt] {$\kappa (e_1) =-2$} (0);
         \draw [dashed, blue] (2) -- node[pos=0.3, font=\footnotesize, fill=white, inner sep=1pt] {$\kappa (e_2) =5$} (0);
         \draw [dashed, blue] (3) -- node[pos=0.65, font=\footnotesize, fill=white, inner sep=1pt] {$\kappa (e_3) =-3$} (0);
         \draw [dashed, blue] (12) -- node[midway, font=\footnotesize, fill=white, inner sep=1pt] {$\kappa (e_4) =0$} (1);
         \draw [dashed, blue] (13) -- node[pos=0.8, font=\footnotesize, fill=white, inner sep=1pt] {$\kappa (e_5) =-2$} (1);
         \draw [dashed, blue] (12) -- node[pos=0.2, font=\footnotesize, fill=white, inner sep=1pt] {$\kappa (e_6) =2$} (2);
         \draw [dashed, blue] (23) -- node[pos=0.7, left, font=\footnotesize, fill=white, inner sep=1pt] {$\kappa (e_7) =3$} (2);
         \draw [thick, red] (13) -- node[pos=0.3, left, font=\footnotesize, fill=white, inner sep=1pt] {$\kappa (e_8) =?$} (3);
         \draw [thick, red] (23) -- node[midway, font=\footnotesize, fill=white, inner sep=1pt] {$\kappa (e_9) =?$} (3);
         \draw [dashed, blue] (123) -- node[pos=0.35, font=\footnotesize, fill=white, inner sep=1pt] {$\kappa (e_{10}) =2$} (12);
         \draw [thick, red] (123) -- node[pos=0.7, font=\footnotesize, fill=white, inner sep=1pt] {$\kappa (e_{11}) =?$} (13);
         \draw [dashed, blue] (123) -- node[pos=0.35, font=\footnotesize, fill=white, inner sep=1pt] {$\kappa (e_{12}) =-3$} (23);
     \end{tikzpicture}
 \end{minipage}
 \hfill 
 \begin{minipage}{0.48\textwidth}
     \centering
     \begin{tikzpicture}[scale=1.2, font=\small] 
         \node (0) at (0,0) {$\emptyset$}; 
         \node (1) at (-2,1.5) {$\{0\}$}; 
         \node (2) at (0,1.5) {$\{1\}$}; 
         \node (3) at (2,1.5) {$\{2\}$}; 
         \node (12) at (-2,3) {$\{0,1\}$}; 
         \node (13) at (0,3) {$\{0,2\}$}; 
         \node (23) at (2,3) {$\{1,2\}$}; 
         \node (123) at (0,4.5) {$\{0,1,2\}$}; 
         \draw [dashed, blue] (1) -- node[pos=0.65, font=\footnotesize, fill=white, inner sep=1pt] {$\kappa (e_1) =-2$} (0);
         \draw [dashed, blue] (2) -- node[pos=0.3, font=\footnotesize, fill=white, inner sep=1pt] {$\kappa (e_2) =5$} (0);
         \draw [dashed, blue] (3) -- node[pos=0.65, font=\footnotesize, fill=white, inner sep=1pt] {$\kappa (e_3) =-3$} (0);
         \draw [dashed, blue] (12) -- node[midway, font=\footnotesize, fill=white, inner sep=1pt] {$\kappa (e_4) =0$} (1);
         \draw [dashed, blue] (13) -- node[pos=0.8, font=\footnotesize, fill=white, inner sep=1pt] {$\kappa (e_5) =-2$} (1);
         \draw [dashed, blue] (12) -- node[pos=0.2, font=\footnotesize, fill=white, inner sep=1pt] {$\kappa (e_6) =2$} (2);
         \draw [dashed, blue] (23) -- node[pos=0.7, left, font=\footnotesize, fill=white, inner sep=1pt] {$\kappa (e_7) =3$} (2);
         \draw [thick, red] (13) -- node[pos=0.3, left, font=\footnotesize, fill=white, inner sep=1pt] {$\kappa (e_8) =?$} (3);
         \draw [dashed, blue] (23) -- node[midway, font=\footnotesize, fill=white, inner sep=1pt] {$\kappa (e_9) =-6$} (3);
         \draw [dashed, blue] (123) -- node[pos=0.35, font=\footnotesize, fill=white, inner sep=1pt] {$\kappa (e_{10}) =2$} (12);
         \draw [thick, red] (123) -- node[pos=0.7, font=\footnotesize, fill=white, inner sep=1pt] {$\kappa (e_{11}) =?$} (13);
         \draw [dashed, blue] (123) -- node[pos=0.35, font=\footnotesize, fill=white, inner sep=1pt] {$\kappa (e_{12}) =-3$} (23);
     \end{tikzpicture}
 \end{minipage}
 
\end{figure}

\begin{figure}[h!] 
 \centering 
 
 \begin{minipage}{0.48\textwidth}
     \centering
     \begin{tikzpicture}[scale=1.2, font=\small] 
         \node (0) at (0,0) {$\emptyset$}; 
         \node (1) at (-2,1.5) {$\{0\}$}; 
         \node (2) at (0,1.5) {$\{1\}$}; 
         \node (3) at (2,1.5) {$\{2\}$}; 
         \node (12) at (-2,3) {$\{0,1\}$}; 
         \node (13) at (0,3) {$\{0,2\}$}; 
         \node (23) at (2,3) {$\{1,2\}$}; 
         \node (123) at (0,4.5) {$\{0,1,2\}$}; 
         \draw [dashed, blue] (1) -- node[pos=0.65, font=\footnotesize, fill=white, inner sep=1pt] {$\kappa (e_1) =-2$} (0);
         \draw [dashed, blue] (2) -- node[pos=0.3, font=\footnotesize, fill=white, inner sep=1pt] {$\kappa (e_2) =5$} (0);
         \draw [dashed, blue] (3) -- node[pos=0.65, font=\footnotesize, fill=white, inner sep=1pt] {$\kappa (e_3) =-3$} (0);
         \draw [dashed, blue] (12) -- node[midway, font=\footnotesize, fill=white, inner sep=1pt] {$\kappa (e_4) =0$} (1);
         \draw [dashed, blue] (13) -- node[pos=0.8, font=\footnotesize, fill=white, inner sep=1pt] {$\kappa (e_5) =-2$} (1);
         \draw [dashed, blue] (12) -- node[pos=0.2, font=\footnotesize, fill=white, inner sep=1pt] {$\kappa (e_6) =2$} (2);
         \draw [dashed, blue] (23) -- node[pos=0.7, left, font=\footnotesize, fill=white, inner sep=1pt] {$\kappa (e_7) =3$} (2);
         \draw [dashed, blue] (13) -- node[pos=0.3, left, font=\footnotesize, fill=white, inner sep=1pt] {$\kappa (e_8) =3$} (3);
         \draw [dashed, blue] (23) -- node[midway, font=\footnotesize, fill=white, inner sep=1pt] {$\kappa (e_9) =-6$} (3);
         \draw [dashed, blue] (123) -- node[pos=0.35, font=\footnotesize, fill=white, inner sep=1pt] {$\kappa (e_{10}) =2$} (12);
         \draw [thick, red] (123) -- node[pos=0.7, font=\footnotesize, fill=white, inner sep=1pt] {$\kappa (e_{11}) =?$} (13);
         \draw [dashed, blue] (123) -- node[pos=0.35, font=\footnotesize, fill=white, inner sep=1pt] {$\kappa (e_{12}) =-3$} (23);
     \end{tikzpicture}
 \end{minipage}
 \hfill 
 \begin{minipage}{0.48\textwidth}
     \centering
     \begin{tikzpicture}[scale=1.2, font=\small] 
         \node (0) at (0,0) {$\emptyset$}; 
         \node (1) at (-2,1.5) {$\{0\}$}; 
         \node (2) at (0,1.5) {$\{1\}$}; 
         \node (3) at (2,1.5) {$\{2\}$}; 
         \node (12) at (-2,3) {$\{0,1\}$}; 
         \node (13) at (0,3) {$\{0,2\}$}; 
         \node (23) at (2,3) {$\{1,2\}$}; 
         \node (123) at (0,4.5) {$\{0,1,2\}$}; 
         \draw [dashed, blue] (1) -- node[pos=0.65, font=\footnotesize, fill=white, inner sep=1pt] {$\kappa (e_1) =-2$} (0);
         \draw [dashed, blue] (2) -- node[pos=0.3, font=\footnotesize, fill=white, inner sep=1pt] {$\kappa (e_2) =5$} (0);
         \draw [dashed, blue] (3) -- node[pos=0.65, font=\footnotesize, fill=white, inner sep=1pt] {$\kappa (e_3) =-3$} (0);
         \draw [dashed, blue] (12) -- node[midway, font=\footnotesize, fill=white, inner sep=1pt] {$\kappa (e_4) =0$} (1);
         \draw [dashed, blue] (13) -- node[pos=0.8, font=\footnotesize, fill=white, inner sep=1pt] {$\kappa (e_5) =-2$} (1);
         \draw [dashed, blue] (12) -- node[pos=0.2, font=\footnotesize, fill=white, inner sep=1pt] {$\kappa (e_6) =2$} (2);
         \draw [dashed, blue] (23) -- node[pos=0.7, left, font=\footnotesize, fill=white, inner sep=1pt] {$\kappa (e_7) =3$} (2);
         \draw [dashed, blue] (13) -- node[pos=0.3, left, font=\footnotesize, fill=white, inner sep=1pt] {$\kappa (e_8) =3$} (3);
         \draw [dashed, blue] (23) -- node[midway, font=\footnotesize, fill=white, inner sep=1pt] {$\kappa (e_9) =-6$} (3);
         \draw [dashed, blue] (123) -- node[pos=0.35, font=\footnotesize, fill=white, inner sep=1pt] {$\kappa (e_{10}) =2$} (12);
         \draw [dashed, blue] (123) -- node[pos=0.7, font=\footnotesize, fill=white, inner sep=1pt] {$\kappa (e_{11}) =1$} (13);
         \draw [dashed, blue] (123) -- node[pos=0.35, font=\footnotesize, fill=white, inner sep=1pt] {$\kappa (e_{12}) =-3$} (23);
     \end{tikzpicture}
 \end{minipage}
 
\end{figure}

\newpage
\begin{algorithm}
  \caption{Adjoint Pass $z_{P} = L_{\mathrm{ext}}^{\top}(q)$ via Direct Transposition}
  \label{alg:adjoint-pass}
  \begin{algorithmic}[1]
    \Require A vector $q \in \mathbb{R}^{|E|}$; forward computation sequence of tree edges
      $\mathcal{S}_{T} = (e_1,\dots,e_k)$; corresponding sequence of vertices
      $(u_1,\dots,u_k)$.
    \Ensure A vector $z_{P} \in \mathbb{R}^{|P|}$.

    \State \textbf{(Initialization of the adjoint state)}
    \State Create a temporary adjoint flow vector $\bar{\kappa} \in \mathbb{R}^{|E|}$.
    \State Initialize it with the input vector $q$ by setting
      $\bar{\kappa}_e \gets q_e$ for all $e \in E$.

    \State \textbf{(Adjoint propagation via reverse iteration)}
    \For{$i \gets k, k-1, \dots, 1$}
      \State Let $e_i \in T$ be the $i$-th tree edge in $\mathcal{S}_{T}$ and
        $u_i \in V$ be the corresponding vertex used in the forward pass.
      \Statex \hspace{\algorithmicindent}The forward computation at step $i$ was
      \Statex \hspace{2\algorithmicindent}%
        $\displaystyle
          \kappa_{e_i}
          =
          -\sigma_{u_i e_i}
          \sum_{e' \in \operatorname{Inc}(u_i) \setminus \{e_i\}}
          \sigma_{u_i e'} \kappa_{e'}.$
      \Statex \hspace{\algorithmicindent}The corresponding adjoint update distributes
        $\bar{\kappa}_{e_i}$ to the contributing edges:
      \ForAll{$e' \in \operatorname{Inc}(u_i) \setminus \{e_i\}$}
        \State
          $\displaystyle
            \bar{\kappa}_{e'}
            \gets
            \bar{\kappa}_{e'}
            - \sigma_{u_i e'} \sigma_{u_i e_i} \bar{\kappa}_{e_i}.$
      \EndFor
    \EndFor

    \State \textbf{(Restriction to $P$)}
    \State Define $z_{P} \in \mathbb{R}^{|P|}$ by
      $(z_{P})_e \gets \bar{\kappa}_e$ for all $e \in P$.

    \State \Return $z_{P}$.
  \end{algorithmic}
\end{algorithm}
\newpage

\begin{algorithm}
  \caption{Full Algorithm for $v \mapsto M^{-1}v$)}
  \label{alg:Minv}
  \begin{algorithmic}[1]
    \State \textbf{Setup (once per IPM iteration)}
    \State Define edge weights $w_e \gets (\max(D,I))_{ee}$ for all $e \in E$.
    \State Use Kruskal's algorithm to find a minimum spanning tree $T$ of
      $\mathcal{G}$ with respect to the weights $w$.
    \State Define the co-tree $P \gets E \setminus T$. This choice of $(T,P)$
      implicitly defines the operators $A_m$ and $\max(D_m, I)$.

    \Statex
    \State \textbf{Application (given input $v \in \mathbb{R}^{|\mathcal{B}|}$)}
    \Require $v \in \mathbb{R}^{|\mathcal{B}|}$
    \Ensure $u_3 = M^{-1} v$
    \State Compute $u_1 \gets A_m^{-T} v$ by applying the sequence of adjoint
      operators
      \[
        u_1
        =
        (L_{\mathrm{ext}})^{\top}
        \circ (K^{-1})^{\top}
        \circ R^{\top}(v),
      \]
      using Algorithm~\ref{alg:adjoint-pass} for the action of
      $L_{\mathrm{ext}}^{\top}$.
    \State Compute $u_2 \gets \max(D_m, I)^{-1} u_1$, i.e.\ perform element-wise
      division:
      \[
        (u_2)_e
        =
        \bigl(\max(D_m, I)^{-1}\bigr)_{ee} (u_1)_e
        \quad \text{for all } e \in E.
      \]
    \State Compute $u_3 \gets A_m^{-1} u_2$ by applying the sequence of forward
      operators
      \[
        u_3
        =
        R
        \circ K^{-1}
        \circ L_{\mathrm{ext}}(u_2),
      \]
      using Algorithm~\ref{alg:forward-pass} for the action of $L_{\mathrm{ext}}$.
    \State \Return $u_3$
  \end{algorithmic}
\end{algorithm}

\begin{figure}[h!]
\centering

\begin{subfigure}{0.48\textwidth}
\centering
\begin{tikzpicture}[scale=1.5, font=\small]
    \node (0) at (0,0) {$\emptyset$};
    \node (1) at (-2,1.5) {$\{0\}$};
    \node (2) at (0,1.5) {$\{1\}$};
    \node (3) at (2,1.5) {$\{2\}$};
    \node (12) at (-2,3) {$\{0,1\}$};
    \node (13) at (0,3) {$\{0,2\}$};
    \node (23) at (2,3) {$\{1,2\}$};
    \node (123) at (0,4.5) {$\{0,1,2\}$};

    \graph [edges={-latex}] {
        (1) -> (0); (2) -> (0); (3) -> (0);
        (12) -> (1); (12) -> (2);
        (13) -> (1); (13) -> (3);
        (23) -> (2); (23) -> (3);
        (123) -> (12); (123) -> (13); (123) -> (23);
    };
\end{tikzpicture}
\caption{Boolean lattice graph $\mathcal{G}$ for $n=3$.}
\end{subfigure}
\hfill

\begin{subfigure}{0.48\textwidth}
\centering
\begin{tikzpicture}[scale=1.5, font=\small]
    \node (0) at (0,0) {$\emptyset$};
    \node (1) at (-2,1.5) {$\{0\}$};
    \node (2) at (0,1.5) {$\{1\}$};
    \node (3) at (2,1.5) {$\{2\}$};
    \node (12) at (-2,3) {$\{0,1\}$};
    \node (13) at (0,3) {$\{0,2\}$};
    \node (23) at (2,3) {$\{1,2\}$};
    \node (123) at (0,4.5) {$\{0,1,2\}$};

    \graph [edges={thick, red}] {
        (1) -- (0);
        (12) -- (2); (13) -- (1); (23) -- (3);
        (123) -- (13);
        (13) -- (3);
        (23) -- (2);
    };

    \graph [edges={dashed, blue}] {
        (2) -- (0); 
        (3) -- (0);
        (12) -- (1);
        (123) -- (12); (123) -- (23);
    };
\end{tikzpicture}
\caption{A spanning tree $T$ and its co-tree $P$.}
\end{subfigure}

\caption{Boolean lattice $\mathcal{G}$ and one possible spanning tree/co-tree partition.}
\end{figure}
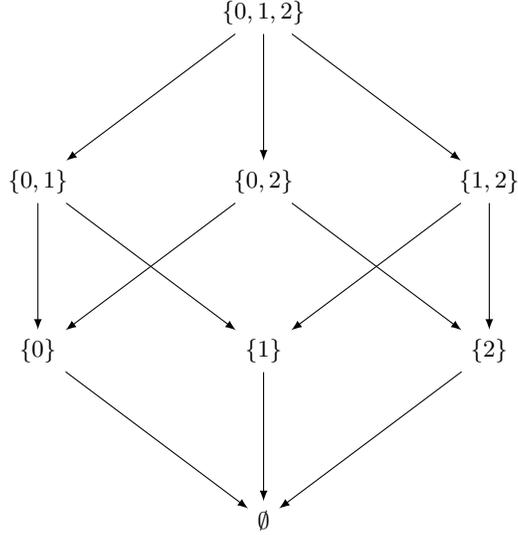
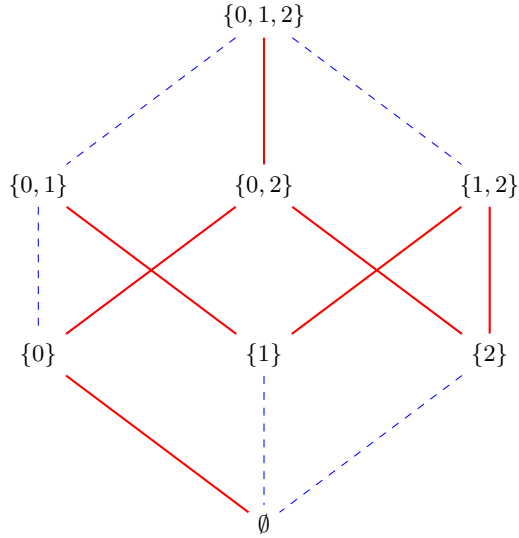

\newpage

\section{Application 1: The Adjoint Method for the
Differentiable Quadratic Program  in Deep Learning}
We consider the quadratic program (QP) defined in the reduced-coordinate space $\xi \in \R^{|\mathcal{B}|}$, parameterized by a vector $\mathbf{c} \in \R^{|\mathcal{B}|}$:
\begin{equation}
\begin{aligned}
\xi^\star(\mathbf{c}) = \arg\min_{\xi} \quad & \frac{1}{2} \xi^T H_0 \xi - \mathbf{c}^T \xi \\
\text{s.t.} \quad & G\xi + \mathbf{b} \ge 0,
\end{aligned}
\label{eq:qp}
\end{equation}
where $H_0 := B^T B$ is symmetric positive semidefinite, $G := KB$, and $\mathbf{b} = Ku$. We assume that the necessary regularity conditions for a unique solution hold.

The KKT conditions for the optimal solution $(\xi^\star, \lambda^\star, s^\star)$ are given by the system $F(\xi, \lambda, s, \mathbf{c}) = 0$:
\begin{subequations}
\begin{align}
    \nabla_{\xi} \mathcal{L} = H_0 \xi^\star - \mathbf{c} - G^T \lambda^\star &= 0 \label{eq:kkt_stationarity} \\
    G\xi^\star + \mathbf{b} - s^\star &= 0 \label{eq:kkt_primal} \\
    \Lambda^\star S^\star \mathbf{1} &= 0 \label{eq:kkt_comp} \\
    s^\star, \lambda^\star &\ge 0,
\end{align}
\label{eq:kkt}
\end{subequations}
where $\mathcal{L}$ is the Lagrangian, and $\Lambda^\star, S^\star$ are diagonal matrices of the vectors $\lambda^\star, s^\star$. The solution $(\xi^\star, \lambda^\star, s^\star)$ is an implicit function of the parameter $\mathbf{c}$.

\begin{theorem}[Adjoint Gradient Computation]
Let $\mathcal{J}(\xi^\star)$ be a scalar loss function that depends on the solution $\xi^\star$ of the QP \eqref{eq:qp}. Let $g_{\xi^\star} := \nabla_{\xi^\star} \mathcal{J}$ be the gradient of the loss with respect to $\xi^\star$. The gradient of the loss with respect to the parameter $\mathbf{c}$, denoted $g_{\mathbf{c}} := \nabla_{\mathbf{c}} \mathcal{J}$, is given by the solution $\mathbf{w}$ to the linear system:
\begin{equation}
    \Hcal(\xi^\star, \lambda^\star, s^\star) \mathbf{w} = g_{\xi^\star}
\end{equation}
where $\Hcal := H_0 + G^T D G$ is the symmetric positive-definite matrix from the reduced interior-point system at convergence, with $D = (S^\star)^{-1} \Lambda^\star$.
\end{theorem}

\begin{proof}
The objective is to compute the total derivative $\frac{d\mathcal{J}}{d\mathbf{c}}$. By the chain rule:
\begin{equation}
    \frac{d\mathcal{J}}{d\mathbf{c}} = \frac{\partial \mathcal{J}}{\partial \xi^\star} \frac{\partial \xi^\star}{\partial \mathbf{c}} = g_{\xi^\star}^T \frac{\partial \xi^\star}{\partial \mathbf{c}}.
\end{equation}
The core task is to find the Jacobian matrix $\frac{\partial \xi^\star}{\partial \mathbf{c}}$. We achieve this by applying the Implicit Function Theorem to the KKT system \eqref{eq:kkt}. We differentiate the KKT conditions with respect to the components of $\mathbf{c}$, holding at the optimal solution.

Let us differentiate the stationarity condition \eqref{eq:kkt_stationarity} and the primal feasibility condition \eqref{eq:kkt_primal}:
\begin{align}
    H_0 \frac{\partial \xi^\star}{\partial \mathbf{c}} - I - G^T \frac{\partial \lambda^\star}{\partial \mathbf{c}} &= 0 \label{eq:diff_stat} \\
    G \frac{\partial \xi^\star}{\partial \mathbf{c}} - \frac{\partial s^\star}{\partial \mathbf{c}} &= 0 \label{eq:diff_primal}
\end{align}
From \eqref{eq:diff_primal}, we have $\frac{\partial s^\star}{\partial \mathbf{c}} = G \frac{\partial \xi^\star}{\partial \mathbf{c}}$.

Next, we differentiate the complementarity condition \eqref{eq:kkt_comp}. At the solution, for each component $i$, either $s_i^\star=0$ or $\lambda_i^\star=0$ (or both, in degenerate cases). Let $\mathcal{A} = \{i \mid s_i^\star=0\}$ be the active set and $\mathcal{I} = \{i \mid \lambda_i^\star=0\}$ be the inactive set.
\begin{itemize}
    \item For $i \in \mathcal{A}$, $s_i^\star=0$. Since $s_i(\mathbf{c}) \ge 0$, any feasible perturbation requires $ds_i \ge 0$, and since $\lambda_i^\star \ge 0$, we have $\lambda_i^\star ds_i = 0$.
    \item For $i \in \mathcal{I}$, $\lambda_i^\star=0$. Any feasible perturbation requires $d\lambda_i \ge 0$, and since $s_i^\star \ge 0$, we have $s_i^\star d\lambda_i = 0$.
\end{itemize}
Differentiating $\Lambda^\star S^\star = 0$ yields $\Lambda^\star \frac{\partial S^\star}{\partial \mathbf{c}} + S^\star \frac{\partial \Lambda^\star}{\partial \mathbf{c}} = 0$, which implies:
\begin{equation}
    \lambda_i^\star \frac{\partial s_i^\star}{\partial \mathbf{c}} + s_i^\star \frac{\partial \lambda_i^\star}{\partial \mathbf{c}} = 0, \quad \forall i. \label{eq:diff_comp}
\end{equation}
For $i \in \mathcal{A}$, $s_i^\star=0$ and $\lambda_i^\star > 0$ (assuming strict complementarity), which forces $\frac{\partial s_i^\star}{\partial \mathbf{c}} = 0$.
For $i \in \mathcal{I}$, $\lambda_i^\star=0$ and $s_i^\star > 0$, which forces $\frac{\partial \lambda_i^\star}{\partial \mathbf{c}} = 0$.

Let's define the diagonal matrix $D = (S^\star)^{-1} \Lambda^\star$. At the solution, for $i \in \mathcal{A}$, $s_i^\star \to 0$ and $\lambda_i^\star > 0$, so $D_{ii} \to \infty$. For $i \in \mathcal{I}$, $\lambda_i^\star = 0$ and $s_i^\star > 0$, so $D_{ii} = 0$.

From \eqref{eq:diff_comp}, we can establish a relationship analogous to the interior-point variable elimination. The stationarity of the IPM Lagrangian implies that $\frac{\partial \lambda^\star}{\partial \mathbf{c}} = -D \frac{\partial s^\star}{\partial \mathbf{c}}$.
Substituting $\frac{\partial s^\star}{\partial \mathbf{c}}$ from \eqref{eq:diff_primal}:
\begin{equation}
    \frac{\partial \lambda^\star}{\partial \mathbf{c}} = -D \left( G \frac{\partial \xi^\star}{\partial \mathbf{c}} \right). \label{eq:dlambda_dc}
\end{equation}

Now, substitute \eqref{eq:dlambda_dc} into the differentiated stationarity condition \eqref{eq:diff_stat}:
\begin{equation}
    H_0 \frac{\partial \xi^\star}{\partial \mathbf{c}} - I - G^T \left( -D G \frac{\partial \xi^\star}{\partial \mathbf{c}} \right) = 0
\end{equation}
Rearranging the terms:
\begin{equation}
    H_0 \frac{\partial \xi^\star}{\partial \mathbf{c}} + G^T D G \frac{\partial \xi^\star}{\partial \mathbf{c}} = I
\end{equation}
\begin{equation}
    (H_0 + G^T D G) \frac{\partial \xi^\star}{\partial \mathbf{c}} = I
\end{equation}
Let $\Hcal := H_0 + G^T D G$, which is exactly the SPD matrix of the reduced IPM system at convergence.
\begin{equation}
    \Hcal \frac{\partial \xi^\star}{\partial \mathbf{c}} = I \quad \implies \quad \frac{\partial \xi^\star}{\partial \mathbf{c}} = \Hcal^{-1}. \label{eq:jacobian_final}
\end{equation}
This gives us the full Jacobian of the QP solution with respect to its parameter vector.
Finally, we compute the desired gradient of the loss, $g_{\mathbf{c}}$:
\begin{equation}
    g_{\mathbf{c}}^T = g_{\xi^\star}^T \frac{\partial \xi^\star}{\partial \mathbf{c}} = g_{\xi^\star}^T \Hcal^{-1}.
\end{equation}
Taking the transpose of both sides (and noting $\Hcal$ is symmetric):
\begin{equation}
    g_{\mathbf{c}} = (\Hcal^{-1})^T g_{\xi^\star} = \Hcal^{-1} g_{\xi^\star}.
\end{equation}
To find $g_{\mathbf{c}}$, we do not invert $\Hcal$. Instead, we solve the linear system:
\begin{equation}
    \Hcal g_{\mathbf{c}} = g_{\xi^\star}.
\end{equation}
This is precisely the relation stated in the theorem, with $\mathbf{w} = g_{\mathbf{c}}$.
\end{proof}

\begin{remark}[Computational Equivalence]
This theorem establishes a profound and computationally convenient result. The gradient required for the backward pass of a deep learning model, $g_{\mathbf{c}}$, is found by solving a linear system with the \textbf{exact same symmetric positive-definite matrix $\Hcal$} that is solved against in the forward pass (the interior-point iterations). Therefore, the same numerical machinery—the Conjugate Gradient method and the custom-designed preconditioner $M$—can be directly reused for both the forward and backward passes, ensuring maximum computational efficiency and code reuse.
\end{remark}
\textbf{Real World Application of the Program}

Unlike vague assessments of whether an output is good or bad the output $\rho^\star$ of our framework possesses a mathematically provable and binary property: theoretical consistency with economic axioms. For any given input $\hat{\rho}$, the constructed output $\rho^\star$ necessarily satisfies all the axioms of a RUM and belongs by construction to the polytope defined by those axioms:
\[
\rho^\star \in \left\{ \rho \mid K\rho \geq 0,\, C\rho = 1 \right\}.
\]

The framework not only outputs a theoretically consistent solution $\rho^\star$ but also yields an endogenous measure of model-theory conflict. Specifically, we quantify how much deviation from the data-driven output $\hat{\rho}$ must be made to satisfy the theoretical constraints. We solve the QP:
\[
\min_{\rho} \|\rho - \hat{\rho}\|_2^2,
\]
where $\hat{\rho} = f_\theta(\mathbf{z})$ is the original prediction produced by a neural network. 

In model diagnosis or risk monitoring, one can track $\mathcal{J}^\star$ continuously. A small value of $\mathcal{J}^\star(f_\theta(\mathbf{z}))$ indicates a high level of alignment between the neural prediction and the axiomatic structure, hence strong confidence in the result. Conversely, a large value of $\mathcal{J}^\star(f_\theta(\mathbf{z}))$ serves as an unambiguous warning signal. It indicates a substantial conflict between the learned pattern from data and the injected axioms, possibly due to distributional shifts, rare macroeconomic anomalies (such as financial crises), or other latent structural deviations. We also show a statistical test mechanism of this measure.

\begin{figure}
    \centering
    \includegraphics[width=0.5\linewidth]{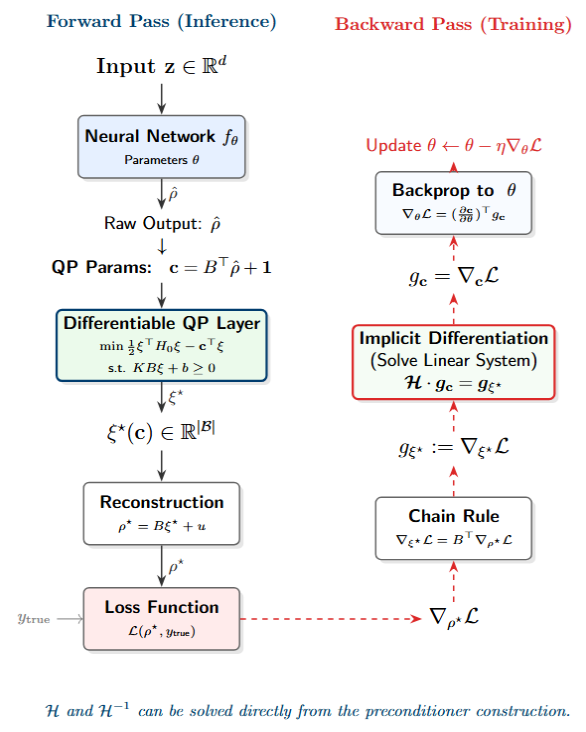}
    \caption{Network Construction}
    \label{fig:placeholder}
\end{figure}

\newpage
\section{Application 2: The Hypothesis Testing Problem and the Boundary Challenge}

The fundamental scientific question is whether an empirically observed choice behavior, summarized by a frequency vector $\hat{\bm{\pi}}$, is consistent with the theory of Random Utility Model.

\begin{definition}[The Null and Alternative Hypotheses]
Let $\bm{\pi}_0$ be the true, population-level choice probability vector, and let $\mathcal{R} = \{ \bm{\rho} \in \mathbb{R}^N : K\bm{\rho} \geq \mathbf{0}, \, C\bm{\rho} = \mathbf{1} \}$ be the RUM-consistent polytope. We test the null hypothesis $H_0$ against the alternative $H_1$:
\begin{equation}
    H_0: \bm{\pi}_0 \in \mathcal{R} \quad \text{versus} \quad H_1: \bm{\pi}_0 \notin \mathcal{R}.
\end{equation}
This is geometrically equivalent to testing if the minimum distance from $\bm{\pi}_0$ to the convex set $\mathcal{R}$ is zero.
\end{definition}

\begin{definition}[The Test Statistic]
Let $\Omega$ be a symmetric positive-definite weighting matrix. The test statistic is defined as the $\Omega$-weighted squared distance from the empirical vector $\hat{\bm{\pi}}$ to the constraint set $\mathcal{R}$:
\begin{equation}
    J_N(\Omega) := N \cdot \min_{\bm{\eta} \in \mathcal{R}} (\hat{\bm{\pi}} - \bm{\eta})^\top \Omega (\hat{\bm{\pi}} - \bm{\eta}).
\end{equation}
\end{definition}

\begin{remark}[The Asymptotic Challenge]
A primary difficulty, as established in the econometrics literature (\cite{AndrewsSoares2010}), is that the limiting distribution of $J_N(\Omega)$ under $H_0$ is a non-standard \emph{chi-bar-squared} ($\bar{\chi}^2$) distribution. This is a weighted sum of chi-squared distributions, where the weights depend on the unknown local geometry of the polytope $\mathcal{R}$ at the true parameter $\bm{\pi}_0$, especially if $\bm{\pi}_0$ lies on the boundary $\partial \mathcal{R}$. This dependency makes the analytical calculation of critical values for the test intractable.
\end{remark}

To surmount this challenge, we employ a bootstrap procedure specifically designed to provide consistent inference in the presence of such boundary conditions.

\subsection{The Solver as a Computational Oracle for Inference}

The statistical procedure requires repeated, efficient computation of Euclidean projections onto $\mathcal{R}$. The QP solver developed in the preceding sections, based on the reparameterization $\bm{\rho} = B \bm{\xi} + \bm{u}$, provides the exact computational oracle for this task.

\begin{definition}[The Euclidean Projection Oracle $\mathbb{P}_{\mathcal{R}}$]
The oracle is a mapping $\mathbb{P}_{\mathcal{R}}: \mathbb{R}^N \to \mathcal{R} \times \mathbb{R}_{+}$ that takes an arbitrary vector $\mathbf{p} \in \mathbb{R}^N$ and returns its L2 Euclidean projection onto $\mathcal{R}$ and the corresponding squared distance.
\begin{equation}
    \mathbb{P}_{\mathcal{R}}(\mathbf{p}) := \left( \bm{\rho}^\star, \|\mathbf{p} - \bm{\rho}^\star\|_2^2 \right), \quad \text{where} \quad \bm{\rho}^\star := \arg\min_{\bm{\eta} \in \mathcal{R}} \|\mathbf{p} - \bm{\eta}\|_2^2.
\end{equation}
\end{definition}

\begin{proposition}[Solver as Oracle Implementation]
Let $\mathcal{A}_{\text{QP}}$ denote the complete interior-point algorithm that solves the QP \eqref{wced}. The oracle $\mathbb{P}_{\mathcal{R}}$ is implemented by the following sequence:
\begin{enumerate}
    \item \textbf{Input Transformation:} Given $\mathbf{p} \in \mathbb{R}^N$, construct the QP linear coefficient vector $\mathbf{c} := B^\top \mathbf{p} + \mathbf{1}$.
    \item \textbf{Core QP Solution:} Invoke the solver to obtain the optimal reduced vector $\bm{\xi}^\star := \mathcal{A}_{\text{QP}}(\mathbf{c})$.
    \item \textbf{Output Reconstruction:} Reconstruct the L2 projection $\bm{\rho}^\star := B\bm{\xi}^\star + \bm{u}$ and compute the squared Euclidean distance $\|\mathbf{p} - \bm{\rho}^\star\|_2^2$.
\end{enumerate}
This proposition establishes the key separation of concerns: the statistical 
inference procedures can be expressed entirely through an oracle interface for optimization subproblems, with this oracle instantiated concretely by our high-performance interior-point solver in the computational 
implementation.

\end{proposition}

\subsection{The Complete Algorithm for a Boundary-Robust Bootstrap Test}

\begin{algorithm}[H]
\caption{A Bootstrap Test for RUM Consistency}
\label{alg:bootstrap_test}
\begin{algorithmic}[1]
\State \textbf{Inputs:} Empirical vector $\hat{\bm{\pi}} \in \mathbb{R}^N$; sample size $N$; number of bootstrap replications $M$; significance level $\alpha$; weighting matrix $\Omega$.

\Statex
\State \textbf{Step 1: Compute the Observed Test Statistic}
\State Compute the L2 projection of the observed data by calling the oracle:
    $ (\bm{\rho}_{\text{obs}}^\star, \_) \gets \mathbb{P}_{\mathcal{R}}(\hat{\bm{\pi}}) $.
\State Compute the $\Omega$-weighted test statistic:
    $ J_N \gets N \cdot (\hat{\bm{\pi}} - \bm{\rho}_{\text{obs}}^\star)^\top \Omega (\hat{\bm{\pi}} - \bm{\rho}_{\text{obs}}^\star) $.

\Statex
\State \textbf{Step 2: Compute the Tightened Centering Point}
\State Select a fixed point $\bm{\rho}_{\text{int}} \in \operatorname{int}(\mathcal{R})$ (e.g., constructed from $\delta_{\text{int}}(D,x)=1/|D|$).
\State Set the tightening rate $\tau_N = c N^{-a}$ for pre-specified constants $c > 0$ and $a \in (0, 1/2)$.
\State Define the tightening vector: $\bm{s}_N := \tau_N \bm{\rho}_{\text{int}}$.
\State Compute the projection of the shifted data: $(\bm{\rho}_{\text{tight}}^\star, \_) \gets \mathbb{P}_{\mathcal{R}}(\hat{\bm{\pi}} - \bm{s}_N)$.
\State Form the final centering point, which is guaranteed to be in $\operatorname{int}(\mathcal{R})$:
    $ \bm{\eta}_{\text{center}} \gets \bm{\rho}_{\text{tight}}^\star + \bm{s}_N $.

\Statex
\State \textbf{Step 3: Execute the Bootstrap Loop}
\State Initialize the set of bootstrap statistics $\mathcal{J}^\star \gets \emptyset$.
\For{$m = 1, \dots, M$}
    \State Generate a bootstrap sample: $\hat{\bm{\pi}}^{\star} \sim \frac{1}{N} \cdot \text{Multinomial}(N, \hat{\bm{\pi}})$.
    \State Recenter the bootstrap sample to simulate the null hypothesis:
        $ \tilde{\bm{\pi}}^\star \gets \hat{\bm{\pi}}^\star - \hat{\bm{\pi}} + \bm{\eta}_{\text{center}} $.
    \State Compute the L2 projection for the recentered sample:
        $ (\bm{\rho}_{\text{boot}}^\star, \_) \gets \mathbb{P}_{\mathcal{R}}(\tilde{\bm{\pi}}^\star) $.
    \State Compute the bootstrap test statistic with the same weighting matrix $\Omega$:
        $ J_N^\star \gets N \cdot (\tilde{\bm{\pi}}^\star - \bm{\rho}_{\text{boot}}^\star)^\top \Omega (\tilde{\bm{\pi}}^\star - \bm{\rho}_{\text{boot}}^\star) $.
    \State $\mathcal{J}^\star \gets \mathcal{J}^\star \cup \{J_N^\star\}$.
\EndFor

\Statex
\State \textbf{Step 4: Make Statistical Decision}
\State Compute the bootstrap p-value from the empirical distribution $\mathcal{J}^\star$:
    $ \hat{p} \gets \frac{1}{M} \sum_{J^\star \in \mathcal{J}^\star} \mathbf{1}\{ J^\star \geq J_N \}$.
\If {$\hat{p} < \alpha$}
    \State \textbf{return} Reject $H_0$.
\Else
    \State \textbf{return} Do not reject $H_0$.
\EndIf
\end{algorithmic}
\end{algorithm}

\section{Extension to Incomplete Data: Low-Rank Perturbation and Spectral Acceleration}

In practical applications, choice probabilities are frequently observed only for a subset of all possible choice sets. We formulate the projection problem under incomplete data and demonstrate that, counter-intuitively, data sparsity induces a low-rank structure in the Hessian that significantly accelerates the solver convergence.

\subsection{Canonical Formulation on the Observation Manifold}

Let $\mathfrak{D} = 2^X \setminus \emptyset$ be the set of all possible choice sets. We define the observation domain $\mathcal{D} \subseteq \mathfrak{D}$ as the subset of choice sets for which empirical data is available. The observation manifold $\mathcal{M}$ is defined as:
\begin{equation}
    \mathcal{M} := \bigl\{ (D,x) \in \mathcal{D} \times X : x \in D \bigr\}.
\end{equation}
Let $N = |\mathfrak{D}|$ and $M = |\mathcal{M}|$. We introduce the canonical projection operator $P_{\mathcal{M}}: \mathbb{R}^N \to \mathbb{R}^N$ which masks unobserved entries:
\begin{equation}
    (P_{\mathcal{M}} v)(D,x) := 
    \begin{cases}
    v(D,x) & \text{if }(D,x)\in \mathcal{M}, \\
    0 & \text{otherwise}.
    \end{cases}
\end{equation}
Note that $P_{\mathcal{M}}$ is idempotent ($P_{\mathcal{M}}^2 = P_{\mathcal{M}}$) and symmetric, acting as an orthogonal projector onto the subspace of observed coordinates.

Given an empirical vector $\hat{\rho} \in \mathbb{R}^N$ (where $\hat{\rho}$ is supported on $\mathcal{M}$), the primal problem is to find $\rho^* \in \mathbb{R}^N$ solving:
\begin{equation}
    \min_{\rho \in \mathbb{R}^N} \frac{1}{2} \| P_{\mathcal{M}}(\rho - \hat{\rho}) \|_2^2 \quad \text{subject to} \quad \rho = B\xi + u, \quad K\rho \ge 0.
\end{equation}
Substituting the affine parameterization $\rho = B\xi + u$ leads to the unconstrained quadratic formulation in the reduced variable $\xi \in \mathbb{R}^d$:
\begin{equation} \label{eq:inc_obj}
    \min_{\xi \in \mathbb{R}^d} \Psi(\xi) := \frac{1}{2} \xi^\top Q_{\mathcal{M}} \xi + c_{\mathcal{M}}^\top \xi,
\end{equation}
where the Hessian operator $Q_{\mathcal{M}}$ and linear term $c_{\mathcal{M}}$ are given by:
\begin{equation}
    Q_{\mathcal{M}} := B^\top P_{\mathcal{M}} B, \quad c_{\mathcal{M}} := B^\top P_{\mathcal{M}} (u - \hat{\rho}).
\end{equation}

\subsection{Spectral Decomposition of the Data Hessian}

The convergence of the primal-dual interior point method is governed by the spectral properties of the Newton system matrix $H_k$ at iteration $k$:
\begin{equation}
    H_k = Q_{\mathcal{M}} + (KB)^\top D_k (KB),
\end{equation}
where $D_k$ is the strictly positive diagonal barrier matrix. We explicitly characterize the spectrum of $Q_{\mathcal{M}}$.

\begin{lemma}[Block-Diagonal Decomposition] \label{lemma:block}
    The matrix $Q_{\mathcal{M}} \in \mathbb{R}^{d \times d}$ admits a block-diagonal decomposition up to permutation. For each observed set $D \in \mathcal{D}$, let $n_D = |D|-1$. The sub-block corresponding to $D$, denoted $W_{n_D} \in \mathbb{R}^{n_D \times n_D}$, satisfies:
    \begin{equation}
        W_{n_D} = I_{n_D} + \mathbf{1}_{n_D} \mathbf{1}_{n_D}^\top,
    \end{equation}
    where $I_{n_D}$ is the identity matrix and $\mathbf{1}_{n_D}$ is the vector of all ones. For any $D \notin \mathcal{D}$, the corresponding block is the zero matrix.
\end{lemma}

\begin{proof}
    The operator $B$ maps reduced coordinates $\xi$ to flow conservation residuals. For any $(D,x) \in \mathcal{M}$, the action of $P_{\mathcal{M}} B$ is restricted to $D$. The quadratic form decomposes as:
    \begin{equation}
        \xi^\top B^\top P_{\mathcal{M}} B \xi = \sum_{D \in \mathcal{D}} \left\| (B\xi)_D \right\|_2^2.
    \end{equation}
    For a fixed $D$, let $\xi_D \in \mathbb{R}^{n_D}$ be the associated variables. The term corresponds to $\sum_{i=1}^{n_D} \xi_i^2 + (\sum_{i=1}^{n_D} \xi_i)^2 = \xi_D^\top (I + \mathbf{1}\mathbf{1}^\top) \xi_D$.
\end{proof}

\begin{proposition}[Rank and Inertia] \label{prop:rank}
    The rank of the Hessian $Q_{\mathcal{M}}$ is strictly determined by the degrees of freedom of the observed data:
    \begin{equation}
        r := \operatorname{rank}(Q_{\mathcal{M}}) = \sum_{D \in \mathcal{D}} (|D| - 1).
    \end{equation}
    Furthermore, the non-zero eigenvalues of $Q_{\mathcal{M}}$ are given by:
    \begin{equation}
        \sigma(Q_{\mathcal{M}}) \setminus \{0\} = \bigcup_{D \in \mathcal{D}} \left( \{1\}^{|D|-2} \cup \{|D|+1\}^1 \right).
    \end{equation}
\end{proposition}

\begin{proof}
    The matrix $W_m = I_m + \mathbf{1}_m \mathbf{1}_m^\top$ is a rank-1 perturbation of the identity. Its eigenvalues are $\lambda_1 = 1 + \mathbf{1}^\top \mathbf{1} = m+1$, and $\lambda_i = 1$ for $i=2,\dots,m$. The result follows from the block-diagonal structure established in Lemma \ref{lemma:block}.
\end{proof}

\subsection{Perturbation Theory and Convergence Bounds}

We now analyze the preconditioned operator $G_k = M_k^{-1/2} H_k M_k^{-1/2}$. Assuming the preconditioner $M_k$ efficiently captures the barrier term (i.e., $M_k \approx (KB)^\top D_k (KB)$), the spectral analysis reduces to studying the perturbation of the identity by the transformed data Hessian.

Let $G \approx I + \tilde{Q}$, where $\tilde{Q} = M_k^{-1/2} Q_{\mathcal{M}} M_k^{-1/2}$. By Sylvester's Law of Inertia, $\operatorname{rank}(\tilde{Q}) = \operatorname{rank}(Q_{\mathcal{M}}) = r$.

 \begin{theorem}[Spectral Clustering of the Preconditioned Operator]
    Let $\Lambda(G)$ denote the spectrum of the preconditioned linear system. Under the rank condition from Proposition \ref{prop:rank}, the eigenvalues satisfy:
    \begin{equation}
        \dim \ker (G - I) \ge d - r.
    \end{equation}
    Specifically, the spectrum consists of a cluster at unity with multiplicity at least $d-r$, and at most $r$ outliers strictly greater than 1:
    \begin{equation}
        \Lambda(G) \subset \{1\} \cup [\mu_{\min}, \mu_{\max}], \quad \text{where } \mu_{\min} > 1.
    \end{equation}
\end{theorem}

\begin{proof}
    Let $G = I + E$, where $E$ is positive semidefinite with rank $r$. Let $E = U \Sigma U^\top$ be the eigendecomposition. The null space $\mathcal{N}(E)$ has dimension $d-r$. For any $v \in \mathcal{N}(E)$, $Gv = (I+0)v = v$. Thus, 1 is an eigenvalue with geometric multiplicity $d-r$.
\end{proof}

\begin{theorem}[Finite Termination of Krylov Subspace Methods] \label{thm:convergence}
    Consider the Preconditioned Conjugate Gradient (PCG) method applied to the system $G x = b$. Let $e_k$ be the error at step $k$. In exact arithmetic, the method satisfies:
    \begin{equation}
        e_{r+1} = 0.
    \end{equation}
    That is, the solver converges to the exact solution in at most $r+1$ iterations, independent of the ambient dimension $d$.
\end{theorem}

\begin{proof}
    The convergence of CG is bounded by the polynomial approximation problem on the spectrum; see  \cite{LiesenStrakos2012Krylov}:
    \begin{equation}
        \frac{\|e_k\|_G}{\|e_0\|_G} \le \min_{p \in \Pi_k, p(0)=1} \max_{\lambda \in \Lambda(G)} |p(\lambda)|.
    \end{equation}
    Construct the annihilating polynomial $q(\lambda) \in \Pi_{r+1}$:
    \begin{equation}
        q(\lambda) = (1-\lambda) \prod_{j=1}^r \left(1 - \frac{\lambda}{\mu_j}\right),
    \end{equation}
    where $\{\mu_1, \dots, \mu_r\}$ are the outlier eigenvalues. We observe $q(1)=0$ and $q(\mu_j)=0$ for all $j$. Since $\Lambda(G) \subseteq \{1\} \cup \{\mu_j\}$, it follows that $\max_{\lambda \in \Lambda(G)} |q(\lambda)| = 0$. By the optimality of the Krylov subspace projection, the residual must vanish at step $k=r+1$. See \cite{Saad2003Iterative, LiesenStrakos2012Krylov}
.
\end{proof}

\begin{remark}
Theorem \ref{thm:convergence} establishes a scaling law for the solver's performance. The computational complexity is decoupled from the total dimension of the RUM polytope and is instead bounded linearly by the information content of the dataset:
\begin{equation}
    \mathcal{T}_{\text{iter}} \in \mathcal{O}(r) = \mathcal{O}\left( \sum_{D \in \mathcal{D}} (|D|-1) \right).
\end{equation}
Consequently, sparsity in the observation domain $\mathcal{D}$ acts as an implicit dimensionality reduction mechanism. The solver naturally exploits the low-rank structure of $Q_{\mathcal{M}}$ to annihilate the error subspace associated with unobserved constraints, theoretically guaranteeing acceleration as $|\mathcal{D}| \to 0$.
\end{remark}

\section{Numerical Experiments}

We validate our framework through a trilogy of experiments designed to isolate and verify its spectral properties, numerical robustness, and statistical learning dynamics. To reconcile the discrete nature of combinatorial operators with continuous optimization, our codebase employs a hybrid architecture: graph-theoretic operators (e.g., Möbius transforms, tree traversals) are accelerated via Just-In-Time (JIT) compilation (using Numba), while differentiable operations are managed by PyTorch. All experiments were executed in a cloud-based high-performance computing environment (Google Colab), utilizing an NVIDIA A100 Tensor Core GPU to ensure consistent benchmarking of the solver's throughput.

\subsection{Experiment I: Spectral Verification of the Low-Rank Hypothesis}
\label{sec:exp1_spectral}

The theoretical analysis in Section 6 suggests a counter-intuitive scaling law: the computational complexity of the projection step should decrease as the dataset becomes sparser (i.e., as data becomes more incomplete). Theorem \ref{thm:convergence} establishes that the convergence of the Krylov subspace solver is governed by the effective rank $r$ of the observation manifold, rather than the ambient dimension $N$.

To empirically verify this hypothesis decoupling the verification from of a dynamic Interior Point Method (where the barrier matrix $D$ evolves continuously), we design a controlled ``Frozen Barrier'' protocol.

\subsubsection{Experimental Design: The Frozen Barrier Protocol}

We construct a static, highly ill-conditioned environment that simulates the solver's state near convergence. In this regime, the barrier matrix $D$ usually exhibits extreme condition numbers, typically posing severe challenges for iterative solvers. This isolation dissects the spectral contribution of the data term $Q_{\mathcal{M}}$ by ensuring the preconditioner $M$ perfectly handles the barrier term.

\textbf{Protocol Definition.} Let $N$ be the total dimension of the primal space. We select a random active set $\mathcal{A} \subset \{1, \dots, N\}$ representing approximately $80\%$ of the constraints. We define the frozen barrier matrix $D \in \mathbb{R}^{N \times N}$ as a diagonal matrix where:
\begin{equation}
    D_{ii} = \begin{cases}
        10^6 & \text{if } i \in \mathcal{A} \quad (\text{simulating } s_i \to 0), \\
        1 & \text{otherwise}.
    \end{cases}
\end{equation}
We then construct the tree-preconditioner $M$ strictly based on this fixed $D$. Under this construction, the preconditioned operator $G = M^{-1}H$ satisfies:
\begin{equation}
    G = M^{-1}(H_{\text{barrier}} + Q_{\mathcal{M}}) \approx I + M^{-1} Q_{\mathcal{M}}.
\end{equation}
Since $M$ effectively whitens the barrier component ($M^{-1}H_{\text{barrier}} \approx I$), the spectral properties of $G$and consequently the convergence rate of the PCG solver  are dominated solely by the low-rank perturbation $M^{-1} Q_{\mathcal{M}}$.

We fix the problem scale at $n=10$ alternatives ($N \approx 5,120$ variables). We sweep the data availability ratio $\eta = |\mathcal{M}| / |\mathfrak{D}|$ from $1\%$ to $100\%$. For each $\eta$, we solve a synthetic linear system $H x = b$ to a strict relative residual tolerance of $10^{-10}$.

\begin{algorithm}[h]
\caption{Spectral Verification Protocol (Frozen Barrier)}
\label{alg:spectral_test}
\begin{algorithmic}[1]
\Require Problem size $n=10$, Sparsity sweep points $\{\eta_1, \dots, \eta_k\} \subset (0, 1]$.
\Ensure Relation between Effective Rank $r$ and PCG Iterations.

\State \textbf{Phase 1: Freeze the Barrier}
\State Initialize full indexer for dimension $N = n 2^{n-1}$.
\State Select active set $\mathcal{A}$ randomly with $|\mathcal{A}| \approx 0.8 N$.
\State Set $D_{ii} \leftarrow 10^6$ if $i \in \mathcal{A}$, else $1$. \Comment{Inject artificial ill-conditioning}
\State Construct Tree-Preconditioner $M$ using $D$. \Comment{Perfect barrier preconditioning}
\State Define static Barrier Operator $H_{\text{bar}} = (KB)^\top D (KB)$.

\Statex
\State \textbf{Phase 2: Sweep Data Sparsity}
\For{each $\eta \in \{\eta_1, \dots, \eta_k\}$}
    \State Sample observation mask $\mathcal{M}$ such that $|\mathcal{M}| \approx \eta |\mathfrak{D}|$.
    \State Compute Effective Rank $r \leftarrow \sum_{E \in \mathcal{M}} (|E|-1)$.
    \State Define Data Hessian $Q_{\mathcal{M}} = B^\top P_{\mathcal{M}} B$.
    \State Define Total System $H_{\text{total}} = H_{\text{bar}} + Q_{\mathcal{M}}$.
    \State Generate random target $b$ and solve $H_{\text{total}} x = b$ via PCG using $M$.
    \State Record iteration count $k_{\text{iter}}$ required for residual $< 10^{-10}$.
\EndFor
\end{algorithmic}
\end{algorithm}

\subsubsection{Results and Analysis: Linearity and Spectral Clustering}

The results of the frozen barrier experiment are visualized in Figure \ref{fig:spectral_verification}. The empirical data strongly corroborates the theoretical predictions while revealing deeper structural properties of the RUM polytope.

\textbf{1. Decoupling from Ambient Dimension.}
The iteration count is strictly correlated with the information content of the dataset (the effective rank $r$), independent of the ambient dimension $N$. As the data becomes more incomplete ($\eta \to 0$), the solver accelerates linearly. This confirms that our tree-based preconditioner successfully deflates the high-frequency modes associated with the combinatorial constraints, leaving the Krylov subspace to explore only the subspace spanned by the observed data.

\textbf{2. Favorable Spectral Constants.}
A linear regression of iterations versus rank yields a slope of approximately $0.04$ (iterations per unit rank). While the theoretical worst-case bound for a rank-$r$ perturbation is $r$ iterations,, we observe that for full data ($r \approx 4097$), the solver converges in merely $\sim 210$ iterations. 

This $20\times$ acceleration over the theoretical worst-case bound is attributed to \textbf{spectral clustering}. The eigenvalues of the RUM data Hessian $Q_{\mathcal{M}}$ are not arbitrarily distributed; they cluster around a small set of integers determined by the cardinalities of the choice sets $|D|$. The Preconditioned Conjugate Gradient method is known to annihilate clusters of eigenvalues efficiently. Consequently, the solver does not need to resolve every dimension of the rank-$r$ subspace individually, but rather annihilates error components in blocks corresponding to distinct eigenvalues.

\textbf{Summary.} Experiment I confirms that data sparsity acts as an implicit accelerator. In applications where choice data is sparse (e.g., observing only pairwise comparisons among thousands of products), the solver naturally exploits this low-rank structure, operating at a fraction of the cost required for the full problem.

\begin{figure}[h]
    \centering
    \includegraphics[width=0.9\linewidth]{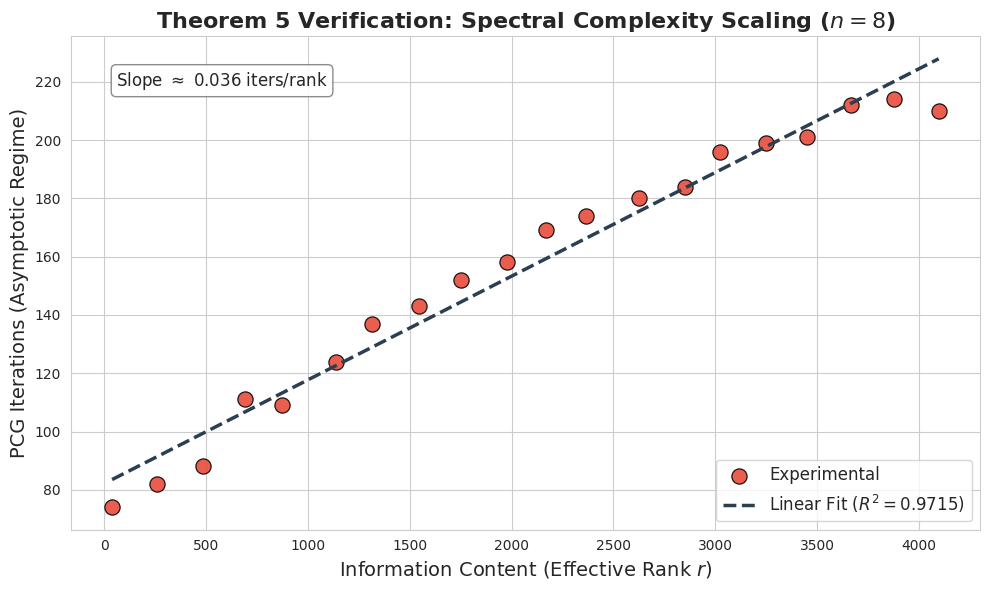}
    \caption{\textbf{Verification of Theorem \ref{thm:convergence}.} Under the Frozen Barrier protocol ($n=10$, ambient dim $\approx 5120$), the PCG iteration count scales linearly with the effective rank of the observed data. The slope is significantly less than 1, indicating strong spectral clustering of the preconditioned operator. }
    \label{fig:spectral_verification}
\end{figure}

\subsection{Experiment II: Numerical Robustness and Convergence Analysis}
\label{sec:exp2_convergence}

While Experiment I validated the theoretical scaling with respect to rank, practical optimization requires robustness against extreme ill-conditioning. In the terminal phase of an Interior Point Method (IPM), as the duality gap closes, the elements of the scaling matrix $D = S^{-1}\Lambda$ diverge. Active constraints force $D_{ii} \to \infty$, while inactive constraints drive $D_{ii} \to 0$. This typically results in a Hessian $H$ with a condition number $\kappa(H)$ exceeding $10^8$, causing standard Krylov solvers to stagnate due to loss of orthogonality and error accumulation.

To demonstrate the numerical stability of our framework, we construct a static ``Stress Test'' representing this worst-case scenario.

\subsubsection{Experimental Setup: The Static Ill-Conditioned Benchmark}

We define a problem instance with $n=8$ ($N=1024$, reduced dimension $d=769$) and construct a static barrier matrix $D$ with a dynamic range of $10^8$. Specifically, we set $D_{ii} = 10^6$ for $80\%$ of the indices (simulating active constraints) and $D_{ii} = 10^{-2}$ for the remainder. We then solve the linear system $H x = b$ using the Conjugate Gradient method under three preconditioning regimes:
\begin{enumerate}
    \item \textbf{Baseline:} No preconditioning (Identity).
    \item \textbf{Jacobi Preconditioner:} A diagonal scaling $M_{\text{Jac}} = \text{diag}(H)$. This is the standard ``cheap'' preconditioner used in large-scale deep learning optimization.
    \item \textbf{Tree-Preconditioner (Ours):} The topological preconditioner $M = A_m^\top \max(D_m, I) A_m$ derived in Section 3.
\end{enumerate}

\textit{Remark on Baselines:} We omit incomplete factorization methods (e.g., Incomplete Cholesky) from this comparison because the Hessian $H$ is accessed solely as an implicit linear operator. Explicit materialization of $H$ for factorization is computationally prohibitive in the high-dimensional RUM settings we target.

\subsubsection{Results and Spectral Analysis}

The convergence history of the residual norm $\|Ax - b\|_2$ is presented in Figure \ref{fig:convergence_benchmark}. The distinct behaviors of the solvers reveal fundamental insights into the spectral structure of the RUM projection problem.

\textbf{1. The Failure of Local Scaling.}
The Jacobi preconditioner (orange curve) performs nearly identically to the unpreconditioned baseline (blue curve), both failing to converge even after 500 iterations. This negative result is mathematically significant. It indicates that the stiffness of the RUM Hessian does not arise principally from variable scaling (which Jacobi corrects perfectly), but from structural coupling. The Boolean lattice constraints introduce dense, non-local dependencies between variables that cannot be resolved by diagonal operators. The condition number remains dominated by the connectivity of the graph, leaving the spectrum spread across eight orders of magnitude.

\textbf{2. Topological Spectral Compression.}
In sharp contrast, our Tree-Preconditioned CG (green curve) exhibits superlinear convergence, reducing the residual by five orders of magnitude ($10^9 \to 10^4$) in fewer than 25 iterations. This ``waterfall'' convergence profile serves as empirical proof of spectral compression. 

Despite the barrier weights varying by factor of $10^8$, the preconditioned operator $M^{-1}H$ possesses a spectrum clustered tightly around unity. By exploiting the Maximum Spanning Tree of the constraint graph, our preconditioner effectively captures the "combinatorial skeleton" of the Hessian. The solver does not see the ill-conditioning of the individual constraints; it sees only the deviation of the full graph from the spanning tree, which is spectrally benign.

\textbf{Implication for Deep Learning.}
This result is decisive for the stability of the \textit{Axioms-as-Layers} paradigm. During backpropagation, the gradient $\nabla_\theta \mathcal{L}$ is computed by solving the linear system $H \cdot w = g$ using the same solver. If the solver stagnates (as seen in the Baseline), the computed gradients will be inaccurate, leading to numerical instability in the neural network training. The rapid convergence of our method ensures that exact gradients are available in constant time, enabling robust end-to-end learning even when the RUM layer operates at the boundary of the feasibility polytope.

\begin{figure}[t]
    \centering
    \includegraphics[width=0.8\linewidth]{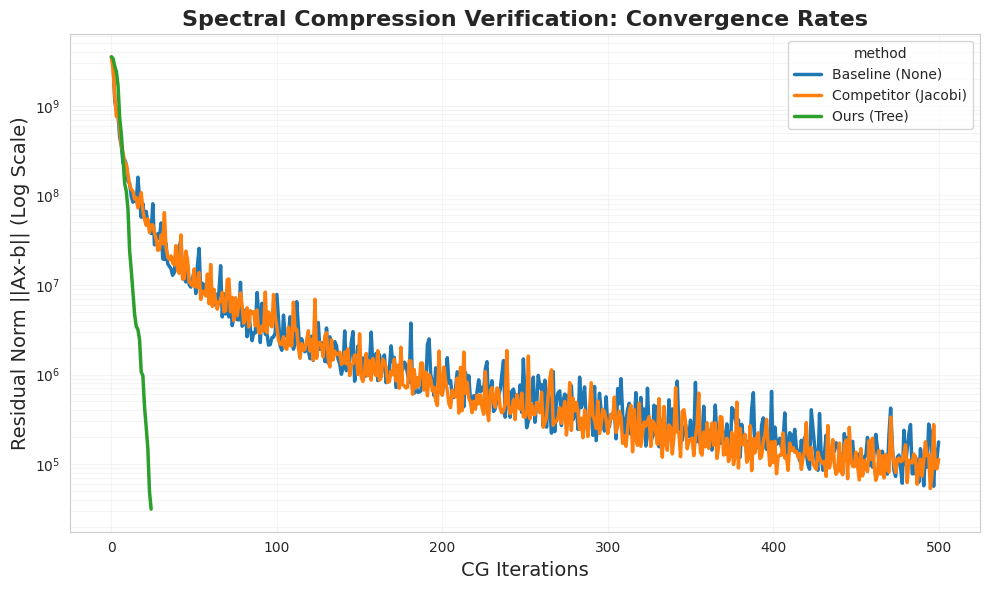}

    \caption{\textbf{Convergence Rate Comparison under Extreme Ill-Conditioning} 
    We compare the residual norm convergence on a static problem where barrier weights vary from $10^{-2}$ to $10^6$. The Baseline and Jacobi methods stagnate, unable to resolve the structural coupling of the constraints. }
    \label{fig:convergence_benchmark}
\end{figure}

\subsection{Experiment III: Inductive Bias and Generalization in Low-Data Regimes}
\label{sec:exp3_generalization}

Having established the spectral efficiency of our solver, we now turn to its implications for statistical learning. A fundamental question in neuro-symbolic AI is how to enforce domain axioms: via soft penalties (Lagrangian relaxation) or hard architectural constraints?

We hypothesize that standard penalty methods suffer from structural overfitting: models may learn to satisfy constraints on training data but fail to generalize the underlying geometric manifold to unseen regions. In contrast, our differentiable projection layer acts as a strong inductive bias, effectively restricting the hypothesis space to the valid RUM polytope $\mathcal{P}$.

\subsubsection{Experimental Setup: The Teacher-Student Protocol}

To evaluate generalization, we employ a Teacher-Student framework where the ground truth is guaranteed to be realizable within the RUM class.

\textbf{Protocol.} 
A "Teacher" network $\mathcal{T}: \mathbb{R}^{d_{ctx}} \to \mathbb{R}^{N}$ maps random contexts to unconstrained scores, which are then projected onto the RUM polytope to generate ground-truth labels $y = \Pi_{\mathcal{P}}(\mathcal{T}(x))$. 
We evaluate on problem scale: $n=8$ ($N=1024$). Crucially, we operate in a data-starved regime with only $N_{\text{train}}=256$ samples, forcing the models to rely on structural priors rather than memorization.

\textbf{Models.} We compare two MLP architectures trained to minimize Mean Squared Error (MSE):
\begin{enumerate}
    \item \textbf{Baseline (Soft Penalty):} The model predicts raw vectors $\hat{y}$. Rationality is encouraged via a penalty term:
    \begin{equation}
        \mathcal{L}_{\text{base}} = \| \hat{y} - y \|_2^2 + \lambda \left( \| \max(0, -K\hat{y}) \|_2^2 + \| C\hat{y} - \mathbf{1} \|_2^2 \right).
    \end{equation}
    We use $\lambda=50$, selected via grid search to balance task loss and constraint violation.
    \item \textbf{Ours (Hard Projection):} The model predicts latent scores $z$, which are passed through our differentiable layer: $\hat{y} = \Pi_{\mathcal{P}}(z)$. The loss is simply $\mathcal{L}_{\text{ours}} = \| \hat{y} - y \|_2^2$.
\end{enumerate}

\subsubsection{Results: Structural Overfitting vs. Geometric Priors}

The learning dynamics are visualized in Figure \ref{fig:generalization}. The divergence between the two approaches reveals a critical limitation of unconstrained learning in combinatorial domains.

\textbf{1. The Safety Gap (Right Plot).}
The Baseline model (Red lines) exhibits a classic failure mode. While it reduces constraint violation on the training set (dashed line), the violation remains high and stagnant on the test set (solid line). This confirms structural overfitting. The network memorizes the specific "safe" coordinates of training examples but fails to learn the geometry of the polytope boundary. 
Conversely, our model (Blue line) maintains machine-precision adherence to the axioms ($\approx 10^{-16}$) across both training and test sets. By solving the KKT system exactly during the forward pass, our model provides a certificate of correctness that is mathematically impossible for soft-penalty methods to guarantee.

\textbf{2. The Generalization Gap (Left Plot).}
The MSE trajectories reveal that embedding axioms improves task performance itself. Our model (Blue) achieves a test loss orders of magnitude lower than the Baseline (Gray). 
This can be explained via \textbf{Hypothesis Space Reduction}. The soft-penalty model must simultaneously learn the task mapping \textit{and} the complex shape of the high-dimensional RUM polytope. Given limited data ($N_{\text{train}}=256$), this problem is under-determined. Our model, however, has the RUM geometry "hard-coded" into its architecture. The gradients backpropagated through the projection layer are inherently geometry-aware, guiding the optimization strictly along the feasible manifold.

\textbf{Summary.}
Experiment III demonstrates that the "Solver-in-the-Loop" paradigm is not merely a mechanism for constraint satisfaction, but a powerful structural regularizer. For high-dimensional combinatorial constraints (like the $n=8$ RUM polytope), learning the geometry from data is inefficient and error-prone; embedding the geometry via our tree-preconditioned solver yields superior generalization and provable safety.

\begin{figure}[t]
    \centering
    \includegraphics[width=1\linewidth]{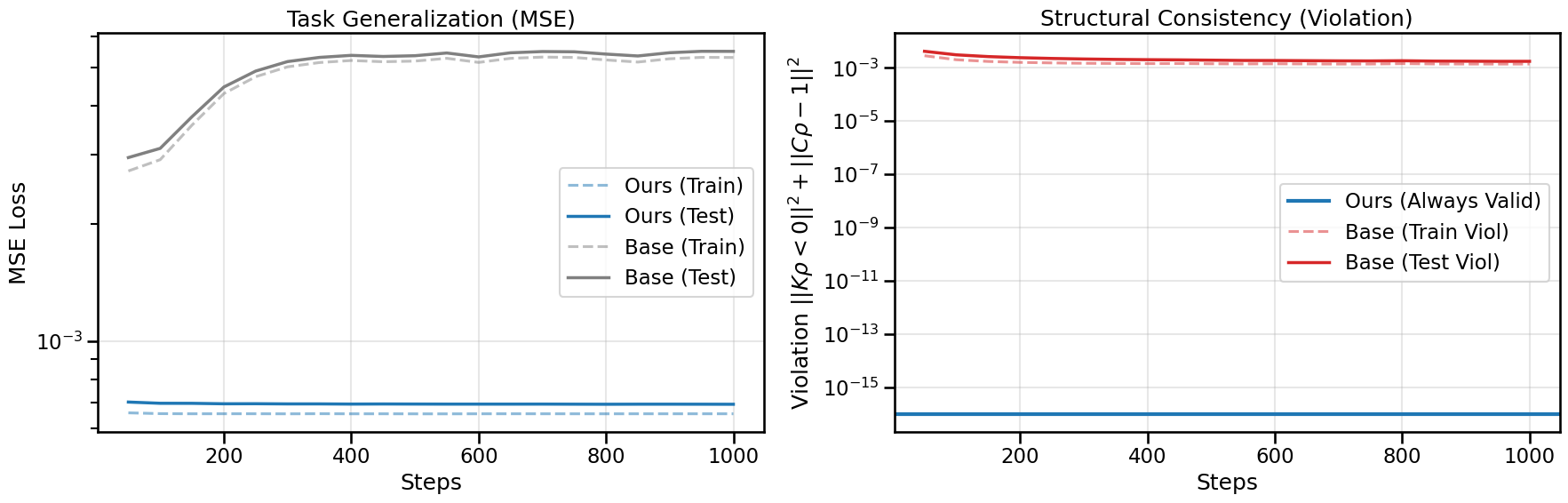}
    \caption{Structural Overfitting in Penalty Methods ($n=8$, $N_{\text{dim}}=1024$). }
    \label{fig:generalization}
\end{figure}

\section{Conclusion and Future Work}

We have established a differentiable optimization framework that rigorously embeds the axiomatic structure of Random Utility Models into the architecture of deep neural networks. By exploiting the geometric duality between the tractable Hyperplane Representation ($\mathcal{P}_H$) and the computationally prohibitive Vertex Representation ($\mathcal{P}_V$), and resolving the resultant numerical ill-conditioning via a novel Tree-Preconditioned Interior Point Method, we have synthesized the normative rigor of economic theory with the representational capacity of modern deep learning.

Our work effectuates a fundamental paradigm shift from the post-hoc verification of rationality to its constructive enforcement. Whereas prior vertex-based approaches (e.g., \cite{SmeuldersCherchyeDeRock2021}) face a factorial complexity barrier ($\Omega(n!)$), our solver achieves polynomial-time scalability by exploiting the latent flow conservation on the boolean lattice. Moreover, by embedding this solver into deep neural networks to enforce 'Axioms-as-Layers,' we enable the end-to-end training of provably rational models at scales $n > 10$ previously deemed inaccessible.

Looking forward, this framework suggests a rigorous trajectory for Neuro-Symbolic AI: the mathematization of domain axioms into differentiable geometric manifolds. Rather than treating rationality as a soft penalty, future work will focus on designing architectural layers that exert \textit{fine-grained control} over neural outputs, strictly confining predictions to the polytope of \textbf{domain rationality}. This approach promises to extend beyond standard utility maximization to encompass complex, non-compensatory decision rules, ultimately yielding learning systems that are not only high-performing but axiomatically correct by construction.

\bibliographystyle{plain} 
\bibliography{references}

\end{document}